%% file: main.tex
\documentclass[11pt]{article}

\usepackage[margin=1in]{geometry}
\usepackage{amsmath,amssymb,amsfonts}
\usepackage{graphicx}
\def\sharedbib{references}
\graphicspath{{./}}
\usepackage{booktabs}
\usepackage[ruled]{algorithm2e}
\usepackage{mdframed}
\usepackage[table]{xcolor}
\usepackage{comment}
\usepackage{subcaption}
\usepackage{threeparttable}
\usepackage{setspace}
\usepackage{booktabs}
\usepackage{multirow}

\usepackage{parskip}

\usepackage[numbers,sort&compress]{natbib}
\usepackage[colorlinks,citecolor=blue,linkcolor=blue,urlcolor=blue]{hyperref}

\SetAlFnt{\small}
\SetAlCapFnt{\small}
\SetAlCapNameFnt{\small}
\SetAlCapHSkip{0pt}
\IncMargin{-\parindent}

\makeatletter
\g@addto@macro\@floatboxreset\centering
\makeatother

\usepackage[suppress]{color-edits}
\addauthor{cp}{red}
\addauthor{am}{blue}
\addauthor{sc}{red}

\newif\ifarxiv
\arxivtrue

\newmdenv[
  linewidth=0.8pt,
  backgroundcolor=gray!5,
  roundcorner=3pt,
  innerleftmargin=8pt,
  innerrightmargin=8pt,
  innertopmargin=8pt,
  innerbottommargin=8pt,
  skipabove=\baselineskip,
  skipbelow=\baselineskip
]{promptbox}

\usepackage[T1]{fontenc}
\usepackage{microtype}
\usepackage{wrapfig}
\usepackage{needspace}
\usepackage{enumitem}
\usepackage{mathtools}
\usepackage{amsthm}
\usepackage[most]{tcolorbox}
\definecolor{ourscolor}{RGB}{230,243,255}
\newtcolorbox{summarybox}{colframe=black!70,boxrule=1.5pt,arc=3pt,left=6pt,right=6pt,top=6pt,bottom=6pt,halign=flush left}

\usepackage[disable,textsize=tiny]{todonotes}

\usepackage[capitalize,noabbrev]{cleveref}
\newcommand{\appendixref}[1]{Appendix~\ref{#1}}

\theoremstyle{plain}
\newtheorem{theorem}{Theorem}[section]
\newtheorem{proposition}[theorem]{Proposition}
\newtheorem{lemma}[theorem]{Lemma}

\theoremstyle{definition}
\newtheorem{definition}[theorem]{Definition}

\theoremstyle{remark}

\title{Where Did We Go Wrong? \\ Localizing Errors to Repair Multi-Turn Conversations}
\author{%
    Alessandro Morosini$^{1}$\quad
    Wei Sun$^{2}$\quad
    Georgia Perakis$^{1}$\\[8pt]
    $^{1}$Massachusetts Institute of Technology\\
    $^{2}$IBM Research
}
\date{\today}

\begin{document}

\maketitle

\begin{abstract}
Large language model (LLM) assistants are increasingly used in multi-turn conversations, where the user's intent is often unclear at first and revealed only gradually over time. 
The assistant must therefore interpret each exchange and form \emph{inferences} about the user's intent, using them to update its \emph{``working state''} (i.e., its accumulated interpretation of the conversation so far).
These updates usually refine the assistant's understanding, but a misinterpretation can instead \emph{corrupt} its working state by ruling out the true intent, influencing subsequent turns and ultimately causing the conversation to fail.
Efficient recovery then requires locating and revising the misinterpretation without discarding information that remains valid.

We introduce \textsc{Loci} (Localization Of Corrupted Inferences), a framework for probabilistically \emph{localizing} misinterpretations and selectively \emph{repairing} the assistant's working state. 
Given an observed failure, \textsc{Loci} reasons over possible user intents and assigns each update a probability of having excluded the true intent.
These probabilities rank suspected updates and quantify the model's confidence in the localization of the underlying misinterpretations. 
\textsc{Loci} then uses this information in a repair policy that decides whether to retain, verify, or remove suspected updates, enabling recovery from corrupted inferences while preserving useful information.

To evaluate \textsc{Loci}'s ability to identify misinterpretations, we first test it in a controlled setup with injected errors. 
We find that \textsc{Loci} improves localization over the baselines, with its advantage becoming larger when multiple errors occur.
We then evaluate it in free-form conversations, where misinterpretations arise naturally, and measure the interactions needed to complete the task after an initial failure. 
Using localization confidence to guide repair, \textsc{Loci} reduces mean recovery cost by 46--79\% relative to the strongest baseline across multiple domains.
\end{abstract}

\input{sections/introduction}
\input{sections/problem_setup}
\input{sections/localization}

\input{sections/results}

\input{sections/related_work}
\input{sections/conclusion}

\bibliography{\sharedbib}
\bibliographystyle{abbrvnat}

\clearpage
\appendix

\renewcommand{\thefigure}{A\arabic{figure}}
\renewcommand{\thetable}{A\arabic{table}}
\setcounter{figure}{0}
\setcounter{table}{0}

\input{appendix/app_proofs}
\input{appendix/app_methodology}
\end{document}

%% file: sections/introduction.tex
\section{Introduction}
\begin{figure}[t]
    \centering
    \includegraphics[width=\linewidth]{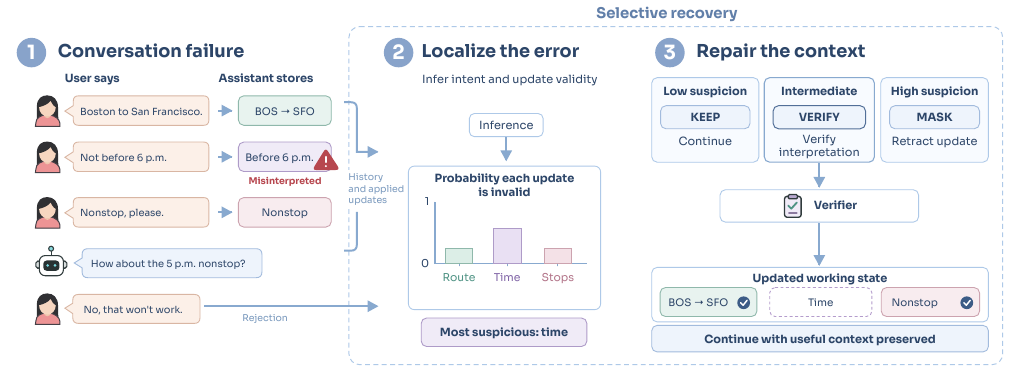}
    \caption{\textbf{Overview of \textsc{Loci}.} (1) The assistant stores an incorrect interpretation of a user constraint, and the error leads to a proposal that the user rejects. (2) \textsc{Loci} uses the rejection and the applied updates to infer the intended target and each update's probability of being invalid. (3) These scores are used to guide a repair policy to selectively fix the conversational state.}
    \label{fig:overview}
\end{figure}

Many real-world applications require large language model (LLM) assistants to engage in a conversation. Users rarely specify all relevant constraints upfront, and the assistant must progressively elicit and refine their intent over multiple exchanges. Yet, while modern LLMs perform strongly on single-turn tasks, their performance has been observed to deteriorate in multi-turn interactions~\citep{kwan-etal-2024-mt, deshpande-etal-2025-multichallenge, laban2026lost}. 
To use the information revealed in each exchange, the assistant must infer what it implies about the user’s intent. When such an inference is wrong, the resulting misinterpretation can become embedded in the assistant’s working state, corrupting its evolving inference and continuing to shape later responses even as new information becomes available.
Such errors can persist across turns and ultimately prevent successful task completion~\citep{kwan-etal-2024-mt, chen-etal-2026-breaking}.
This failure mode more generally appears in sequential decision problems where information is collected over time to infer an unknown state. In a medical context, for example, an early misinterpretation can shift the diagnostic process toward the wrong hypothesis, steering subsequent questions and delaying the correct diagnosis~\citep{fang2026benchmarkingmultiturnmedicaldiagnosis, nori2025sequentialdiagnosislanguagemodels}. %

To address this problem, prior efforts have focused on making LLMs more proactive in preventing misunderstandings, e.g., by planning when and how to gather information \citep{NEURIPS2024_2b0e14ab,dong-etal-2026-value}, or by teaching models to ask clarifying questions and collaborate with users \citep{Chen2025,pmlr-v267-wu25i}. Such approaches can reduce errors, but do not explicitly address the problem of recovery once an error becomes apparent, which may occur several turns after the original mistake, after its effects have already compounded. 
Recent approaches have begun to address recovery in multi-turn interactions by refactoring conversation history \citep{shen2026acr} or generating recovery plans \citep{kim2026rein}. Other work shows that removing prior assistant responses can preserve response quality \citep{huang2026llms}, while modifying selected turns can sometimes improve it \citep{li2026llmsatewordscausal}, suggesting that full conversation history is not always necessary. However, it remains unclear how to effectively identify which earlier turns contain interpretations that caused the failure. 

More broadly, recovery is difficult because the assistant must distinguish useful context from earlier interpretations that may be corrupted. One option is to restart the conversation, discarding all accumulated context: this removes corrupted information, but also useful information, and may lead to the same failure again. At the other end, retaining the full history preserves useful context, but also any earlier misinterpretations. \textbf{Efficient recovery therefore requires identifying which earlier interpretations may have corrupted the conversation while preserving the rest.} This motivates the problem of \emph{localizing corrupted inferences}: identifying the earlier interpretations responsible for a failure so that they can be selectively revised or removed. This capability would be a step toward more reliable context management in long-horizon assistants and agents, which must maintain useful internal state as interactions grow longer and more complex.

In this work we introduce \textsc{Loci} (Localization Of Corrupted Inferences), a framework for probabilistically identifying wrong interpretations in a multi-turn conversation and selectively repairing the conversational state. The key idea is that each candidate user intent implies a possible explanation of which earlier updates were incorrect: if it were the true intent, every update that excludes it would be wrong. \textsc{Loci} weighs these explanations using the recorded updates and the observed failure of the conversation (e.g., user rejects a proposal) to assign each prior update a probability of having excluded the true intent. Because a candidate intent may imply that several earlier updates were wrong, its probability mass contributes to each of their corruption probabilities, allowing \textsc{Loci} to localize multiple errors.

\textsc{Loci} uses the corruption probabilities assigned to earlier updates both to localize likely errors and to guide an adaptive repair policy. For localization, their ranking identifies likely sources of error, while their magnitudes quantify confidence in the attribution. For repair, because retaining invalid updates can prevent recovery while deleting valid updates can require additional questions, the policy uses these scores to decide whether to discard or verify a suspected update, or retain the current state and resume interaction. The policy repeats this decision to recover from the failure, balancing verification costs against the risks of retaining errors or discarding useful information.

We evaluate \textsc{Loci} both for localization accuracy and for recovery efficiency (i.e., number of additional turns required to reach the correct answer after an initial failure). We first study both in a controlled environment where the user's target, the meaning of each update, and the injected errors are known exactly, allowing us to measure localization accuracy explicitly. \textsc{Loci} reliably localizes corrupted updates across both single- and multi-error settings, with its largest gains in interactions containing multiple errors. At the highest tested error rate, \textsc{Loci} ranks an invalid update first in $74.4\%$ of failed interactions, compared with $62.8\%$ for the strongest baseline. We then extend the setup to free-form LLM conversations, where errors arise naturally. \textsc{Loci} reaches the correct guess in $46$--$79\%$ fewer additional turns on average than the strongest evaluated baseline across four domains. We further evaluate \textsc{Loci} on task-oriented dialogue, where it likewise improves localization and recovery. 

In summary, our contributions are as follows. We provide: 
\begin{itemize}[leftmargin=2.5em]
    \item \textbf{A formal model for retrospective error localization.} We represent a conversation as a sequence of (potentially incorrect) updates to the assistant’s working state and formulate the problem of identifying which updates excluded the user’s true intent after an observed failure.
    
    \item \textbf{A probabilistic corruption score for localizing multiple errors.} We assign each update a probability of having excluded the true intent under a specified belief model; these rank likely error locations and quantify the model's confidence in the localization.
    
    \item \textbf{An adaptive policy for selective repair.} We use the localization confidence to decide whether to retain the current state, remove a suspected update, or verify it, balancing the cost of verification against the risks of retaining errors or discarding useful information.
    
    \item \textbf{Empirical evidence of improved localization and recovery.} We evaluate our method in controlled interactions, free-form LLM conversations, and task-oriented dialogue, finding improved error localization and fewer recovery interactions than the evaluated baselines.
\end{itemize}

The rest of the paper is organized as follows. \Cref{sec:problem-formulation} formalizes the problem setup. \Cref{sec:problem} introduces \textsc{Loci} and its repair policy. \Cref{sec:results} presents the empirical evaluation. \Cref{sec:related-work} reviews related work, and \Cref{sec:discussion} concludes.

%% file: sections/problem_setup.tex
\section{Problem Formulation}
\label{sec:problem-formulation}

We formalize a multi-turn interaction as a sequence of potentially incorrect updates to the assistant’s representation of a latent user target. More formally, let $\mathcal{X}$ be a finite candidate set and let $X\in\mathcal{X}$ denote the user's latent target. At turn $i$, the assistant produces a message $a_i$ and the user responds with $u_i$, yielding the history $h_i=(a_1,u_1,\ldots,a_i,u_i)$. The exchange $(a_i,u_i)$ induces an \emph{intended} constraint $\ell_i:\mathcal{X}\to\{0,1\}$, where $\ell_i(x)=1$ means that the candidate $x$ is consistent with the interaction (e.g., a flight departing after the requested time). By construction, $\ell_i(X)=1$ (i.e., the true intent is compatible with every intended constraint). 

The user may state a constraint explicitly or imply it through the exchange and its context. In either case, the assistant must infer how the interaction constrains the latent target. We represent this inference by an \emph{applied} constraint $\tilde{\ell}_i:\mathcal{X}\to\{0,1\}$. An applied update is valid when $\tilde{\ell}_i=\ell_i$ and misinterpreted otherwise. The errors of interest are those that exclude the user’s true intent from the assistant’s state; throughout, we therefore focus on invalid updates satisfying $\tilde{\ell}_i(X)=0$.

The observed conversation and the assistant’s interpretation of it can differ. The history records what was exchanged, while the assistant’s working state records how those exchanges were interpreted. At each turn $i$, we therefore distinguish the raw interaction history $h_i$ from the assistant's working state $C_i=(\tilde{\ell}_1,\ldots,\tilde{\ell}_i)$. The corresponding set of feasible user intents is
\[
    \mathcal{X}_i
    =
    \{x\in\mathcal{X}:
      \tilde{\ell}_j(x)=1 \ \forall j\le i\}.
\]
For example, let $\mathcal{X}$ be a set of possible flights, and suppose the assistant asks $a_i=$ ``What departure time works?'' and the user responds $u_i=$ ``Not before 6 p.m.'' Let $t(x)$ denote the departure time of flight $x$. The intended constraint is $\ell_i(x)=\mathbf{1}\{t(x)\ge 18{:}00\}$, but the assistant instead interprets it as $\tilde{\ell}_i(x)=\mathbf{1}\{t(x)<18{:}00\}$. After this update, $\mathcal{X}_i$ contains only flights departing before 6 p.m. (that also satisfy the earlier constraints). The true flight $X$, which departs after 6 p.m., is therefore erroneously excluded.

To represent whether each applied update correctly reflects the intended constraint, we introduce a latent validity indicator. Let $M_i=\mathbf{1}\{\tilde{\ell}_i=\ell_i\}\in\{0,1\}$ denote the latent validity of update $i$. Since we focus on errors that exclude the target, the indicator satisfies $M_i=\tilde{\ell}_i(X)$.
Let $T$ denote the point in the interaction at which the assistant commits to a candidate $\hat{x}\in\mathcal{X}_T$ based on its current state. If the user accepts the commitment, the interaction succeeds. We focus on the case in which the user rejects it, which signals a failure and triggers repair of the working state. The difficulty is that the rejection shows that the proposed candidate was incorrect, but does not identify which earlier update was invalid. Because $M_{1:T}$ is unobserved, the assistant does not know which earlier interpretations caused the failure. At repair time, the available information is $O_T=(h_T,C_T,\hat{x})$, while the target $X$, intended constraints $\ell_{1:T}$, and validity indicators $M_{1:T}$ are latent.

In the flight example, $\tilde{\ell}_i(X)=0$, so $M_i=0$. Because this update excludes the true flight, $X\notin\mathcal{X}_i$ and remains excluded from $\mathcal{X}_T$. Thus, any candidate $\hat{x}\in\mathcal{X}_T$ proposed at turn $T$ is incorrect and rejected by the user. With multiple preceding updates, however, the rejection does not reveal which $M_j$ are zero, so the source of the failure remains latent at repair time.

%% file: sections/localization.tex
\section{Localization of Corrupted Inferences}
\label{sec:problem}

In this section we introduce \textsc{Loci}, a framework for probabilistically localizing misinterpretations and selectively repairing the assistant's working state. The high level idea is that each candidate user intent implies a possible explanation of which earlier updates were incorrect, and \textsc{Loci} weighs these explanations using the observed updates and failure feedback to assign each update a corruption probability (\Cref{sec:inference}). These probabilities rank likely error locations and quantify the confidence in their localization (\Cref{sec:localization}). A repair policy then uses them to decide whether to verify an update, remove it, or resume interaction (\Cref{sec:repair}). An overview of \textsc{Loci} is provided in \Cref{fig:overview}. 

\subsection{Probabilistic Error Attribution}
\label{sec:inference}

Each candidate intent determines which earlier updates would be invalid: if $x$ were the true target, every update with $\tilde{\ell}_i(x)=0$ would have excluded it. \textsc{Loci} weighs these possible explanations using a belief $b(x)$ over candidate intents after the observed rejection. The corruption score of update $i$ is the belief mass assigned to candidates that it excludes, so an update receives a high score when it rules out plausible explanations of the failure. We first define this score for a given belief $b$ and then specify how $b$ is constructed.

\begin{definition}[\textsc{Loci} score]
\label{def:loci}
Let $b$ be a specified belief distribution over candidate intents in
$\mathcal X$. The \textsc{Loci} score of update $i$ is
\begin{equation}
c_i
:=
\Pr_b(M_i=0)
=
\sum_{x\in\mathcal X}
b(x)\bigl[1-\tilde{\ell}_i(x)\bigr].
\label{eq:corruption}
\end{equation}
Thus, $c_i$ is the corruption probability of update $i$ under $b$.
\end{definition}
To construct $b$, we assign each candidate $x$ a compatibility weight at every update $i$. Let $\rho_i\in(0,1)$ denote an assumed prior reliability of the applied constraint (which may be turn-specific when prior information is available, or common across updates otherwise). Let $q_i(x)\geq 0$ be a specified error weight for candidate $x$. When the error model defines a distribution over incorrect updates, $q_i(x)$ is
the probability it assigns to the recorded applied constraint when $X=x$. If $\tilde{\ell}_i(x)=1$, then $x$ is consistent with the recorded update, which contributes weight $\rho_i$. If $\tilde{\ell}_i(x)=0$, then $x$ can explain the recorded update only if that update was incorrect, which contributes weight $(1-\rho_i)q_i(x)$. Combining the two cases, we define the compatibility weight of candidate $x$ with update $i$ as
\begin{equation}
L_i(x)
:=
\rho_i\tilde{\ell}_i(x)
+
(1-\rho_i)q_i(x)\bigl[1-\tilde{\ell}_i(x)\bigr].
\label{eq:update-likelihood}
\end{equation}
Here, $\rho_i$ represents prior trust in update $i$. Under the resulting belief $b$, the inferred probability that update $i$ is valid is $1-c_i$, which may differ from $\rho_i$. When $q_i(x)$ is constant over incompatible candidates, we write $q_i(x)=\kappa_i$. 

We aggregate the compatibility weights across updates multiplicatively to
define a belief over candidate intents. Given a prior $p_0$ on $X$ and a
rejected commitment $\hat{x}$, let
$S=\mathcal{X}\setminus\{\hat{x}\}$.
Normalizing the prior-weighted compatibility scores over $S$ gives
\begin{equation}
b(x)
:=
\frac{
p_0(x)\prod_{i=1}^{T}L_i(x)
}{
\sum_{x'\in S}p_0(x')\prod_{i=1}^{T}L_i(x')
},
\qquad x\in S,
\label{eq:target-posterior}
\end{equation}
provided the denominator is positive, with $b(x)=0$ outside $S$.
Substituting this belief into \Cref{eq:corruption} gives the
\textsc{Loci} scores under the specified model.

\subsection{Error Localization}
\label{sec:localization}
We localize errors by ranking updates according to their \textsc{Loci}
scores $c_i$. The top-ranked update is
\[
\hat\tau \in \arg\max_{1\leq i\leq T} c_i.
\]
Because $c_i=\Pr_b(M_i=0)$, this rule selects the update with the highest
probability of being invalid under $b$. We use \(c_{\hat\tau}\) as the model's confidence in the error localization.

Each score $c_i$ aggregates belief mass over every candidate intent for which update $i$ is invalid, including candidates that violate several applied constraints.
A candidate excluded by multiple updates
contributes its belief mass to each of their scores. Because multiple updates
may be invalid, the \textsc{Loci} scores do not generally sum to one; their
sum is the expected number of invalid updates under $b$. Once $b$ is
computed, all $T$ scores require only summation.

The single-error case gives a simple interpretation of this localization rule. If exactly one update is invalid, removing it makes the target consistent with every remaining applied constraint. For each update \(i\), define
\[
N_i :=
\left|
\left\{
x\in S:
\tilde\ell_i(x)=0,\;
\tilde\ell_j(x)=1 \text{ for every } j\ne i
\right\}
\right|,
\]
so that $N_i$ counts the candidate intents consistent with update $i$ being
the only invalid update. The following lemma states when these counts determine the conditional distribution over the location of the error.

\begin{lemma}[Single-error localization]
\label{lem:single-error-localization}
Let \(E=\{\sum_{j=1}^{T}(1-M_j)=1\}\) denote the event that exactly one
update is invalid, and let \(\tau\) denote its position. Suppose \(p_0(x)\)
is constant and positive on \(S\), \(\rho_i=\rho\in(0,1)\) for every update
\(i\), and there is a common \(\kappa\in(0,1]\) such that
\(q_i(x)=\kappa\) whenever \(x\in S\) and \(\tilde\ell_i(x)=0\). If
\(N:=\sum_{j=1}^{T}N_j>0\), then
\[
\Pr_b(\tau=i\mid E)=\frac{N_i}{N}.
\]
\end{lemma}
Thus, under the lemma's assumptions,
\(\hat\tau_{\mathrm{SE}}\in\arg\max_i N_i\) is the maximum a posteriori (MAP) localization rule under $b$, conditional on exactly one invalid update. We refer to this estimator as Counting in the experiments and prove the result in \appendixref{sec:proof-single-error}. 
More generally, under the same assumptions, the normalized \textsc{Loci} scores \(c_i/\sum_j c_j\) converge to \(N_i/N\)
as \(1-\rho\) approaches zero, without conditioning on exactly one invalid update
(\appendixref{sec:proof-rare-error}).

\subsection{Selective Repair Policy}
\label{sec:repair}
Localization identifies which updates are most likely to be invalid, but does
not by itself determine how they should be handled. Moreover, if a
misinterpretation is embedded in the assistant's working state, it may
continue to shape its later reasoning, including its attempts to diagnose the
failure in the conversation. This motivates separating task execution from
repair, using a controller with policy $\pi$ to intervene on the working state
of a downstream assistant $\mu$.

After a rejected commitment, repair must balance retaining invalid updates against removing valid ones, since either can increase the interactions required for recovery. Let $N_{\mathrm{repair}}(\pi;\mu)$ denote the number of additional 
interactions required before $\mu$ commits to the correct target $X$, with
$N_{\mathrm{repair}}(\pi;\mu)=\infty$ if recovery does not occur.
For a fixed $\mu$, we define the repair objective as
\[
\pi^\star
\in
\arg\min_{\pi}
\mathbb{E}\!\left[
N_{\mathrm{repair}}(\pi;\mu)
\mid O_T
\right].
\]
Repair should preserve useful information while removing enough of the
corrupted state to recover. Rather than either resuming from the state
unchanged or restarting the conversation from scratch, the controller can
use the \textsc{Loci} scores to intervene selectively on individual updates.
We therefore consider three repair actions: $\texttt{KEEP}$
preserves the current working state and resumes the task;
$\texttt{VERIFY}(i)$ uses one interaction to check whether the target
satisfies the selected constraint, retaining it if valid and removing it
otherwise; $\texttt{MASK}(i)$ removes the constraint from the working state without a user interaction. In our implementation, verification is reliable and is implemented by asking the user to confirm the constraint, although it could also be implemented in alternative ways (e.g., deferring to a stronger model or a task-specific tool). Questions, verification checks, and rejected commitments each cost one interaction, while masking and the final correct commitment cost zero. While we evaluate recovery using additional interactions, the objective can
instead reflect other action costs by redefining $N_{\mathrm{repair}}$
(e.g., latency or tokens).

Masking saves the cost of verification, but is risky when the selected update may still be valid. Verification avoids this risk by checking the update before removing it, but incurs an additional interaction. Intuitively, the controller should therefore mask only at high localization confidence, verify when the evidence is less decisive, and keep the state unchanged when there is little evidence of an error. We model this tradeoff with a two-threshold policy over the highest-scoring unresolved update. At each repair step, let $\hat\tau$ denote the unresolved update with the largest
\textsc{Loci} score. For
\(0 \leq \gamma_{\mathrm{keep}} < \gamma_{\mathrm{mask}} \leq 1\),
the policy takes
\[
\pi_\gamma =
\begin{cases}
\texttt{KEEP},
    & c_{\hat\tau} < \gamma_{\mathrm{keep}},\\[1mm]
\texttt{VERIFY}(\hat\tau),
    & \gamma_{\mathrm{keep}} \leq c_{\hat\tau}
      < \gamma_{\mathrm{mask}},\\[1mm]
\texttt{MASK}(\hat\tau),
    & c_{\hat\tau} \geq \gamma_{\mathrm{mask}}.
\end{cases}
\]
After $\texttt{MASK}$ or $\texttt{VERIFY}$, repair continues over the remaining unresolved updates. Masking removes the selected constraint from the working state while leaving the controller's belief unchanged, since no new information is observed. Verification additionally provides information about the validity of the selected update, which may be incorporated into subsequent
belief updates. When no unresolved updates remain, the controller resumes $\mu$.

%% file: sections/results.tex
\section{Experimental Results}
\label{sec:results}

We evaluate \textsc{Loci} on both error localization and recovery efficiency,
measuring how well it identifies invalid updates and how many additional
interactions are required to recover after an initial failure. 
We start from a simplified, controlled setup with a fixed pool of binary questions and injected mistakes. Knowing which updates are incorrect lets us measure localization accuracy directly, while controlling the repair actions lets us separate the effects of localization and action selection on recovery cost. Then, we relax these simplifications and move to a free-form conversation experiment, where a language model generates questions and interprets responses, allowing mistakes to arise naturally. Finally, we extend the evaluation beyond binary questions to categorical requirement updates in task-oriented dialogues. \appendixref{app:results} provides further experimental details and results.

\subsection{Controlled Experiments}
\label{sec:results-controlled}

\begin{figure}[t]
  \centering
  \includegraphics[width=0.9\linewidth]{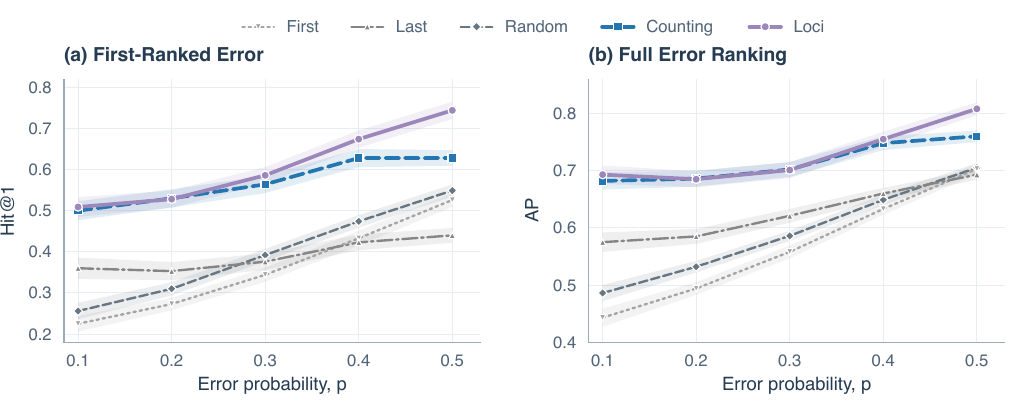}
  \caption{\textbf{Error localization in the controlled 20Q setup.} We compare localization methods after a wrong guess for different error probabilities. (a) Hit@1 measures whether the top-ranked update is invalid. (b) Average precision evaluates the ranking of all invalid updates. }
  \label{fig:controlled-localization}
\end{figure}

\textbf{Setup.}
We start with a controlled environment adapted from the 20 Questions (20Q) game \citep{NEURIPS2024_2b0e14ab}. The assistant aims to identify a hidden target by asking yes/no questions from a fixed pool. Each answer narrows the candidate set (e.g., answering ``yes'' to ``Can it fly?'' retains only candidates that can fly). 
The question--answer pair therefore defines the intended constraint $\ell_i$, which retains candidates whose answers agree with the user's response.

At each turn, with probability $p\in\{0.1,0.2,0.3,0.4,0.5\}$, we simulate a misinterpretation by applying the truthful answer to a different question's predicate, producing an applied constraint $\tilde{\ell}_i$. We uniformly sample this replacement among predicates that exclude the target while retaining at least one current candidate. This gives us invalid updates whose locations are known for evaluating localization and repair. The assistant then continues asking questions until it makes a wrong guess or a later update leaves no feasible candidates, revealing that an earlier interpretation was incorrect.

We evaluate this setup on four main domains: Animals, Food, Objects, and Places \citep{NEURIPS2024_2b0e14ab}. To evaluate performance with a larger set of intents, we also use CUB-200-2011, a dataset of 200 bird species \citep{wah2011caltech} (more details in \appendixref{app:repair-cub}). Given the conversational state, questions are selected from the pool under three policies: \emph{(i)} Random samples uniformly among unasked questions that split the remaining candidates; \emph{(ii)} Value of information (VoI) selects the question with the largest expected information gain \citep{dong-etal-2026-value}; \emph{(iii)} Natural uses Llama 3.1 8B \citep{grattafiori2024llama} to select a question from the pool. Throughout the controlled experiments, the reported results are obtained under the Natural policy, while results under different policies can be found in \appendixref{app:policies} and \appendixref{app:recovery-policies}.

\begin{wrapfigure}[]{r}{0.45\textwidth}
  \centering
  \includegraphics[width=\linewidth]{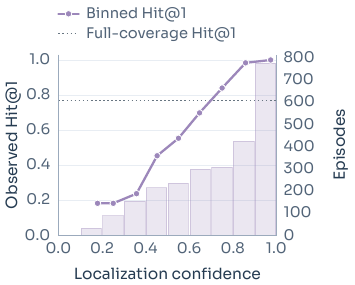}
  \caption{\textbf{Higher localization confidence corresponds to higher accuracy.} The curve shows Hit@1 within bins of $c_{\hat{\tau}}$. Bars show the number of episodes in each bin.
  }
  \label{fig:reliability-coverage}
\end{wrapfigure}

\paragraph{Localization.}
\Cref{fig:controlled-localization} reports Hit@1 in panel (a), which measures whether the top-ranked update is actually invalid, and average precision (AP) in panel (b), which measures how well the ranking places invalid updates ahead of valid ones. We compare \textsc{Loci} with Random, First, and Last, which rank updates randomly, earliest first, and latest first, respectively. We also compare it to Counting, which corresponds to the single-error MAP rule from \Cref{sec:localization}, which ranks updates by the number of candidates made feasible by removing each update. 

\textsc{Loci} and Counting perform similarly at low error rates. Their gap
widens at higher error rates, where failed conversations contain more invalid
updates on average: the mean number of invalid updates per failed conversation
rises from 1.26 at \(p=0.1\) to 2.70 at \(p=0.5\). This pattern is consistent with \textsc{Loci}'s ability to account for multiple invalid updates. At
\(p=0.5\), \textsc{Loci} achieves 74.4\% Hit@1 and 0.808 AP, compared with
62.8\% and 0.760 for Counting, and outperforms all remaining baselines on
both metrics. Full localization results for all three question policies on 20Q are reported in \appendixref{app:policies}, with the CUB setup and results discussed in \appendixref{app:repair-cub}.

Unlike the ranking baselines, which provide only an ordering over updates,
\textsc{Loci} also assigns the selected update a probabilistic score
\(c_{\hat{\tau}}\). On top of localization accuracy, we therefore also evaluate whether this score provides a useful measure of confidence in the top-ranked
update. \Cref{fig:reliability-coverage} groups top-ranked updates by their score and reports Hit@1 within each bin. We observe that accuracy increases monotonically across score bins, indicating that the score helps identify when the top-ranked update is more likely to be invalid. This supports the use of \(c_{\hat{\tau}}\) as a confidence signal
for selective repair.

\textbf{Repair.}
Repair begins after a wrong guess or an empty candidate set. During recovery,
we use the same error-injection procedure and set \textsc{Loci}'s prior update
reliability to \(\rho=1-p\). Positive verification conditions the controller's
belief on the selected constraint being valid. After a negative verification,
we remove the selected constraint from the working state but do not condition
the belief on the negative verdict. In the binary setting, doing so would also
reveal the correct alternative. This convention applies to all policies that
use verification. Recovery cost counts additional questions, verifications,
and rejected guesses after the initial failure; masking and the final correct
guess cost zero. For each domain and error rate, we select
\(\gamma_{\mathrm{keep}}\) and \(\gamma_{\mathrm{mask}}\) from a fixed grid
using five target folds. For each held-out fold, we choose the threshold pair
with the lowest mean recovery cost on the other four folds and evaluate it on
the held-out fold. All repetitions of a target remain in the same fold.
\input{tables/recovery_summary}

To isolate the effect of localization, we first let each method select a
suspected invalid update and use an always-\texttt{VERIFY} policy, which
applies \texttt{VERIFY} to the selected update at every repair step. This
allows us to compare \textsc{Loci} to the baselines under the same repair
action. We then fix localization to \textsc{Loci} and compare three repair policies: always-\texttt{VERIFY}, always-\texttt{MASK}, and the adaptive policy from
\Cref{sec:repair}. We also evaluate a Restart baseline, which discards the current working state and restarts from scratch, retaining only the information that the rejected
guess is not the target. 

Results are shown in \Cref{tab:recovery-summary}. Under the shared
always-\texttt{VERIFY} policy, \textsc{Loci} requires 12.04 recovery
interactions on average, compared with 13.48 for Counting, a reduction of
10.7\%, and also outperforms the remaining localization baselines. Adaptive
repair further reduces mean recovery cost to 9.99 interactions, a 17.1\%
reduction relative to always-\texttt{VERIFY}. Full repair results for all three question policies on 20Q are reported in \appendixref{app:recovery-policies}. Ablations of the repair policy are reported in \appendixref{app:controller-ablations}.

\subsection{Free-Form LLM Interactions}
\label{sec:free-form}
\label{sec:results-freeform}

\paragraph{Setup.}
So far, we have evaluated \textsc{Loci} in a controlled 20Q setting with a fixed question pool and injected errors. We now relax both simplifications while retaining the same underlying task, allowing questions and misinterpretations to arise from free-form language-model interactions. We use Llama 3.1 8B to generate the assistant's questions and to simulate the user's responses. To represent each exchange as a candidate-set update, the assistant model evaluates the question for every candidate and retains
candidates whose predicted answers agree with the recorded response. 
These model-generated constraints define the working state and the updates evaluated by \textsc{Loci}. Recovery begins after a wrong guess. After a repair action, the assistant continues generating questions and interpreting responses using the same procedure. Recovery costs follow \Cref{sec:results-controlled}, with a cap of 100 interactions. For these recovery experiments, we fix the common prior update reliability at \(\rho=0.9\) across all four domains.

\paragraph{Localization.}
Free-form conversations lack ground-truth error labels, making direct evaluation of localization accuracy difficult. An LLM judge can provide reference labels, but these depend on the judge's accuracy and interpretation of ambiguous questions. We therefore use recovery cost under a fixed repair policy as our primary measure of localization quality. All methods start from the same failed conversations and reconstructed updates and use the same always-\texttt{VERIFY} and continuation procedures, so differences in recovery
cost reflect differences in localization decisions.

\textsc{Loci} requires fewer interactions than all evaluated baselines in every
domain. Relative to the strongest baseline in each domain, it reduces mean
recovery cost by 35--50\%. As an additional evaluation, we also compare the localization rankings with reference labels provided by an independent LLM judge. Under this evaluation, \textsc{Loci} achieves 75.0\% Hit@1 and 0.643 AP (\appendixref{app:free-form}).

We also test whether the interacting LLM can diagnose its own errors. This test is particularly interesting in the free-form setting because, unlike with injected errors, the candidate-set updates and resulting misinterpretations are produced by the interacting model itself.
For this, we use Llama 3.1 8B to rank which earlier updates are most likely to require repair and compare it with \textsc{Loci} under the same always-\texttt{VERIFY} policy. 
\textsc{Loci} reduces mean recovery cost by $41.1\%$ on average (\appendixref{app:llama-selector}). This may reflect that the model's original misinterpretations also affect its ability to diagnose the problem: the same model constructs the updates and later judges which ones were incorrect. This result supports decoupling task execution from repair, instead of relying on the same model that produced the failed conversation to diagnose its own errors.

\paragraph{Repair.}
We next fix localization to use \textsc{Loci} scores and compare the always-\texttt{VERIFY}, always-\texttt{MASK}, and the adaptive policy. As in the controlled experiments, we use five target folds within each domain and select the thresholds for each held-out fold using the other four. The adaptive policy reduces mean recovery cost by 12--57\% relative to always-\texttt{VERIFY} and by 15--33\% relative to always-\texttt{MASK} across the four domains.
All three policies recover every failed conversation within the interaction budget. Per-domain results are reported in \Cref{tab:recovery-summary}. 

We also observe larger aggregate gains from using \textsc{Loci} in the free-form setting when compared to the
controlled experiments, for both localization and repair. To investigate this difference, we examine the recovery cost components in \appendixref{app:freeform-repair}. We find that verification accounts for a
larger share of recovery cost under \textsc{Loci} with
always-\texttt{VERIFY} in the free-form setting. Reducing verification checks
can therefore save a larger fraction of the total cost. The adaptive policy
avoids such checks through masking, although it incurs more questions and
rejected guesses than always-\texttt{VERIFY}. This suggests that \textsc{Loci} may be more effective in longer conversations, where it provides the opportunity to save checks while retaining enough information to identify the target. 

As a robustness check, we repeat the free-form recovery evaluation with Qwen3-8B and observe similar recovery gains: \textsc{Loci} with always-\texttt{VERIFY} again outperforms the
localization baselines, and the adaptive policy achieves the lowest mean cost
in all four domains (\Cref{tab:freeform-repair-qwen}). Separately, we ablate the simulated user's access to the conversation history in \appendixref{app:user-history}.

\subsection{Task-Oriented Dialogue}
\label{sec:results-categorical}

\paragraph{Setup.}
We next relax the assumption that questions must receive binary answers. For this, we use MultiWOZ 2.4 \citep{ye2022multiwoz}, a dataset of task-oriented dialogues in which users communicate requirements to an assistant (e.g., finding a cheap hotel in the east). These requirements are annotated as categorical slot values, where each slot is a named requirement (e.g., price or area). Each exchange defines an intended constraint containing all requirements introduced in that exchange (e.g., \texttt{price} $=$ \texttt{cheap} and \texttt{area} $=$ \texttt{east}). The assistant's goal is to recover the complete requirement configuration from a finite set of configurations supported by the database. We exclude annotated configurations outside the candidate space, where successful recovery would be impossible, and require at least two turns that introduce requirements, so localization involves choosing among multiple updates.

Similarly to the 20Q setup, we construct the updates from the annotations and inject errors independently with different probabilities. For each update, we randomly select one slot and replace its value with a uniformly sampled alternative (e.g., changing \texttt{area = east} to \texttt{area = west}). \appendixref{app:multiwoz-setup} details the benchmark, setup, and instantiation of \textsc{Loci}.

\begin{wraptable}[]{R}{0.5\textwidth}
  \centering
  \input{tables/multiwoz_summary}
    \caption{\textbf{\textsc{Loci} improves localization and reduces recovery cost on MultiWOZ.}
    Baselines use the always-\texttt{VERIFY}
    policy, \textsc{Loci} is adaptive.}
  \label{tab:multiwoz-summary}
\end{wraptable}

\textbf{Localization.}
As in \Cref{sec:results-controlled}, we evaluate localization using Hit@1 and AP over invalid updates. \textsc{Loci} achieves higher Hit@1 and AP than all baselines (\Cref{tab:multiwoz-summary}). Compared with Counting, the strongest baseline, it improves Hit@1 from 65.5\% to 75.8\% and AP from 0.817 to 0.871. These results average across error rates, while the full results can be found in \appendixref{app:multiwoz}.

\textbf{Repair.}
During repair, the model selects a slot and generates a question about it; the simulated user returns the slot's annotated value.
Adaptive \textsc{Loci} requires 2.63 recovery interactions on average, versus 3.17 for the strongest baseline (\Cref{tab:multiwoz-summary}). With localization fixed to \textsc{Loci}, always-\texttt{MASK} requires fewer interactions than adaptive repair (\Cref{tab:multiwoz-recovery}). This may be because reliable, targeted slot questions make discarded information cheap to recover, so verification checks cost more than they save in questions and wrong guesses. However, adaptive repair achieves the lowest cost in the free-form experiments, suggesting that the preferred intervention depends not just on localization confidence but also the cost of recovering discarded information, which is lower in this scenario (see \appendixref{app:multiwoz-repair} for more details).

%% file: tables/recovery_summary.tex
\begin{table}[t]
\centering\fontsize{9}{9.5}\selectfont\setlength{\tabcolsep}{1.5pt}
\begin{tabular}{@{}lrrrrrrr@{}}
\toprule
 & \multicolumn{2}{c}{Controlled recovery} & \multicolumn{5}{c}{Free-form LLM recovery} \\
\cmidrule(lr){2-3}\cmidrule(lr){4-8}
Method & 20Q & CUB & Animals & Food & Objects & Places & 20Q \\
\midrule
First + \texttt{VERIFY} & $15.66 \pm 0.50$ & $26.72 \pm 0.44$ & $22.29 \pm 1.79$ & $20.62 \pm 1.38$ & $25.33 \pm 1.76$ & $20.27 \pm 1.77$ & $22.13 \pm 0.87$ \\
Last + \texttt{VERIFY} & $13.84 \pm 0.43$ & $25.81 \pm 0.42$ & $19.71 \pm 1.28$ & $18.57 \pm 1.12$ & $23.67 \pm 1.62$ & $18.73 \pm 1.57$ & $20.17 \pm 0.72$ \\
Random + \texttt{VERIFY} & $14.81 \pm 0.46$ & $26.31 \pm 0.42$ & $20.67 \pm 0.94$ & $19.95 \pm 0.94$ & $26.00 \pm 1.97$ & $20.91 \pm 1.50$ & $21.88 \pm 0.71$ \\
Counting + \texttt{VERIFY} & $13.48 \pm 0.44$ & $23.82 \pm 0.41$ & $21.52 \pm 2.34$ & $18.62 \pm 0.90$ & $24.75 \pm 2.02$ & $19.00 \pm 1.56$ & $20.97 \pm 0.91$ \\
Restart & $63.76 \pm 3.33$ & $428.35 \pm 12.51$ & $43.67 \pm 5.90$ & $46.71 \pm 6.88$ & $56.88 \pm 7.71$ & $20.09 \pm 6.47$ & $41.84 \pm 3.38$ \\
\midrule
\textsc{Loci} + \texttt{VERIFY} & $12.04 \pm 0.38$ & $21.86 \pm 0.36$ & $12.05 \pm 1.34$ & $11.19 \pm 1.03$ & $15.29 \pm 1.41$ & $9.36 \pm 1.22$ & $11.97 \pm 0.63$ \\
\textsc{Loci} + \texttt{MASK} & $30.18 \pm 0.74$ & $217.01 \pm 3.30$ & $12.52 \pm 3.32$ & $10.48 \pm 2.64$ & $15.42 \pm 3.08$ & $6.00 \pm 2.23$ & $11.10 \pm 1.44$ \\
\textsc{Loci} adaptive & $\mathbf{9.99 \pm 0.35}$ & $\mathbf{20.51 \pm 0.35}$ & $\mathbf{10.62 \pm 2.50}$ & $\mathbf{7.10 \pm 1.40}$ & $\mathbf{10.83 \pm 2.13}$ & $\mathbf{4.00 \pm 1.48}$ & $\mathbf{8.14 \pm 0.96}$ \\
\bottomrule
\end{tabular}
\caption{\textbf{Fitted adaptive \textsc{Loci} has the lowest mean recovery cost in both the controlled and free-form setup.} Mean additional interactions after an initial failure $\pm$ target-cluster bootstrap SE. First, Last, Random, and Counting use \texttt{VERIFY}. Free-form 20Q averages the four domains equally.}
\label{tab:recovery-summary}
\end{table}

%% file: tables/multiwoz_summary.tex
\fontsize{8}{10}\selectfont\setlength{\tabcolsep}{2.5pt}
\begin{tabular}{@{}lrrr@{}}
\toprule
 & \multicolumn{2}{c}{Localization} & Repair \\
\cmidrule(lr){2-3}\cmidrule(l){4-4}
Method & Hit@1 $\uparrow$ & AP $\uparrow$ & Interactions $\downarrow$ \\
\midrule
First & $0.530 \pm 0.011$ & $0.745 \pm 0.006$ & $3.51 \pm 0.05$ \\
Last & $0.513 \pm 0.012$ & $0.734 \pm 0.007$ & $3.54 \pm 0.06$ \\
Random & $0.528 \pm 0.005$ & $0.741 \pm 0.004$ & $3.53 \pm 0.06$ \\
Counting & $0.655 \pm 0.012$ & $0.817 \pm 0.006$ & $3.33 \pm 0.05$ \\
Restart & -- & -- & $3.17 \pm 0.06$ \\
\midrule
\textsc{Loci} & $\mathbf{0.758 \pm 0.011}$ & $\mathbf{0.871 \pm 0.006}$ & $\mathbf{2.63 \pm 0.06}$ \\
\bottomrule
\end{tabular}

%% file: sections/related_work.tex
\section{Related Work}
\label{sec:related-work}

\paragraph{Multi-turn failure of LLMs.}
LLMs have been observed to often fail in multi-turn conversations. \citet{laban2026lost} show that LLMs commit to early incorrect interpretations of underspecified requests, struggling to recover as the conversation unfolds. Similar weaknesses are documented by other works such as \citep{kwan-etal-2024-mt, deshpande-etal-2025-multichallenge}. \citet{tack2026llms} show that performance can also degrade when the user revises their intent over successive turns, a failure mode closely related to the explicit corrections documented in \citet{chen-etal-2026-breaking}. \citet{huang2026llms} show that all the assistant's own prior turns can sometimes be omitted from context without degrading the conversation quality. Recent work examines the causal effects of earlier assistant turns \citep{li2026llmsatewordscausal} and organizes nonlinear dialogue histories to better manage context \citep{hu-etal-2026-context}. Overall, these studies suggest that LLMs may over-condition on past (wrong) turns and that revisiting conversation history may help.

\paragraph{Error localization and repair.}
Recovering (efficiently) from a failed conversation requires identifying the earlier updates responsible for the failure. LLMs, however, struggle to self-correct without external feedback \citep{ICLR2024_8b4add8b}, but can often correct an error when given its location \citep{tyen-etal-2024-llms}. In stepwise reasoning, process reward models are used to score intermediate steps \citep{lightman2024let}; closer to us, \citet{kim2026latent} treat step correctness as latent and infer it from the reasoning trace. In a conversation, the bad update may be the interpretation of an earlier exchange. \citet{xie-etal-2022-correctable} study this problem in dialogue states; \citet{shen2026acr} refactor conversation history to mitigate state drift. Other methods plan recovery or repair known memory faults \citep{kim2026rein,yu2026faultymemoriescorrectedactions}. \textsc{Loci} localizes likely invalid updates and makes repair selective by using failure feedback to estimate the validity of each past update.

\paragraph{Clarification and information acquisition.}
Clarification works study what information the assistant should aim to collect, and when asking the user is worth the effort. \citet{rao-daume-iii-2018-learning} rank questions by their expected information value, and \citet{NEURIPS2024_2b0e14ab} plan information-seeking questions under uncertainty. Deciding when to ask also benefits from learned policies and models of future turns \citep{Chen2025,ICLR2025_97e2df4b}. \citet{pmlr-v267-wu25i} frame eliciting user intent as part of long-term collaboration, while \citet{dong-etal-2026-value} weigh the benefit of asking against user effort. Other works test whether models acquire enough information to determine an answer \citep{NEURIPS2025_c42c8d51, huang2026llms-ask}. Our controller weighs the cost of checking a past interpretation against the risks of keeping or removing it.

%% file: sections/conclusion.tex
\section{Discussion}
\label{sec:discussion}

In this work we introduce \textsc{Loci}, a framework for localizing corrupted inferences in multi-turn conversations and selectively repairing the assistant's working state while preserving useful information. 
We first evaluate \textsc{Loci} in a controlled scenario to be able to isolate and directly measure its localization accuracy, and then extend the evaluation to more realistic scenarios: free-form LLM conversations, and task-oriented dialogues. Across all the setups, \textsc{Loci} improves in both localization, as measured by the ability to identify invalid updates, and recovery, as measured by the number of turns needed until successful task completion. Interestingly, we also observe that the same LLM that produced the misinterpretations performs substantially worse than \textsc{Loci} at locating them, even when explicitly asked to diagnose the
failed conversation. Overall, our results suggest that error localization and repair can benefit from a controller separate from the interacting model. More broadly, the ability to identify and selectively revise corrupted interpretations offers a step toward more reliable context management in long-horizon assistants and agents.

\paragraph{Limitations and Future Work.}
Our paper is subject to some limitations. First, we assume the user has a fixed target within a finite candidate set containing the true intent. While this is natural for structured tasks where candidate intents can be constructed from combinations of known attributes (e.g., dialogue slots or tool arguments), this setup is more restrictive if the possible intents are not known in advance or cannot be easily enumerated. One possible extension is to generate or retrieve candidate intents during the conversation, as discussed in \appendixref{app:candidate-intents}, but we leave this for future work.
Second, the assumption of a fixed target excludes interactions where the user revises their preferences, in which case a rejected proposal may reflect a change in the user's preferences instead of a misinterpretation.
Finally, our method requires specifying the prior reliability of the updates, but we observe that this choice has limited influence on recovery cost, which remains similar across the tested values (\appendixref{app:reliability-sensitivity}).
Beyond relaxing these assumptions, future work could explore interactions between \textsc{Loci}
and the underlying LLM. One intuitive way to do so would be providing the scores as additional context to the model that is trying to self-diagnose the misinterpretations. Some of our experiments also suggest that some scenarios might benefit from accounting for the future cost of recovering from discarded information, so this could be incorporated explicitly in the policy. Another interesting direction is to treat information acquisition and repair jointly, by learning a policy that accounts for the possibility of revising earlier steps when deciding what questions to ask next. This could encourage the assistant to ask more exploratory questions, knowing that wrong interpretations can later be identified and revised.

%% file: appendix/app_proofs.tex
\section{Theoretical Results}
\label{app:proofs}

Throughout this section, let $S:=\mathcal X\setminus\{\hat x\}$.

\subsection{Single-Error Localization}
\label{sec:proof-single-error}

\begin{proof}[Proof of Lemma~\ref{lem:single-error-localization}]
For each update $i$, define its single-error witness set
\[
W_i
:=
\{x\in S:\tilde\ell_i(x)=0,\ \tilde\ell_j(x)=1
\text{ for every }j\neq i\}.
\]
Thus, $|W_i|=N_i$. Under the target-excluding error restriction in
\Cref{sec:problem-formulation}, $M_j=\tilde\ell_j(X)$, so $E$ holds exactly
when $X$ violates one applied constraint and satisfies the rest. Since
$b$ is supported on $S$, $E$ is equivalent under $b$ to
$X\in\bigcup_{j=1}^T W_j$. The
sets $W_1,\ldots,W_T$ are
disjoint, and conditional on $E$, the event $\tau=i$ is equivalent to
$X\in W_i$.

For every $x\in W_i$, \Cref{eq:target-posterior} assigns the
unnormalized belief weight
\[
p_0(x)(1-\rho)\kappa\rho^{T-1}.
\]
This weight is the same positive constant for every single-error witness because
$p_0(x)$ is constant and positive on $S$. Therefore,
\[
\Pr_b(\tau=i\mid E)
=
\frac{|W_i|}{\sum_{j=1}^T|W_j|}
=
\frac{N_i}{N},
\]
where the denominator is positive by assumption.
\end{proof}

\subsection{Rare-Error Extension}
\label{sec:proof-rare-error}

The preceding proof conditions on exactly one invalid update. We now consider
the unrestricted belief as the assumed error parameter $p=1-\rho$ approaches zero,
holding the applied constraints, candidate set $S$, and error weights fixed.

\begin{proposition}[Rare-error limit]
\label{prop:rare-error-limit}
Suppose $p_0(x)$ is constant and positive on $S$. Fix $\kappa>0$, and for each $p\in(0,1)$
let $\rho_i=1-p$ and $\kappa_i=\kappa$ for every update $i$. Let $c_i(p)$
denote the corruption probability in
\Cref{eq:corruption} under these parameters, and assume
$N:=\sum_{j=1}^T N_j>0$. Then $\sum_{j=1}^T c_j(p)>0$, and
\[
\bar c_i(p)
:=
\frac{c_i(p)}{\sum_{j=1}^T c_j(p)}
=
\frac{N_i}{N}+O(p)
\qquad\text{as }p\downarrow0
\]
for every update $i$. Moreover, for every pair $i,j$ with $N_i>N_j$,
\[
c_i(p)>c_j(p)
\qquad\text{whenever}\qquad
0<p<\frac{1}{1+\kappa|S|}.
\]
\end{proposition}

\begin{proof}
For $x\in S$, let
\[
d(x)
:=
\sum_{t=1}^T\bigl(1-\tilde\ell_t(x)\bigr),
\qquad
q_p:=\frac{p\kappa}{1-p}.
\]
A satisfied constraint contributes weight $1-p$, while a violated
constraint contributes weight $p\kappa$. Since $p_0(x)$ is constant and positive on $S$,
\Cref{eq:target-posterior} becomes
\[
b_p(x)
=
\frac{q_p^{d(x)}}{Z_p},
\qquad
Z_p
:=
\sum_{x'\in S}q_p^{d(x')}.
\]
The assumption $N>0$ implies that $S$ is nonempty, so $Z_p>0$.

Define
\[
A_i(q)
:=
\sum_{x\in S:\,\tilde\ell_i(x)=0}q^{d(x)}.
\]
Then \Cref{eq:corruption} gives $c_i(p)=A_i(q_p)/Z_p$. The coefficient of $q$ in $A_i(q)$ is $N_i$
because the corresponding candidates violate update $i$ and no other update.
Since $S$ is finite,
\[
A_i(q)=N_iq+O(q^2).
\]
Moreover,
\[
D(q)
:=
\sum_{j=1}^T A_j(q)
=
\sum_{x\in S}d(x)q^{d(x)}
=
N q+O(q^2).
\]
For $q>0$, the assumption $N>0$ gives $D(q)>0$. Hence
$\sum_j c_j(p)=D(q_p)/Z_p>0$, and
\[
\bar c_i(p)
=
\frac{A_i(q_p)}{D(q_p)}
=
\frac{N_i+O(q_p)}{N+O(q_p)}
=
\frac{N_i}{N}+O(q_p).
\]
Finally, $q_p=p\kappa/(1-p)=O(p)$ as $p\downarrow0$, which proves the
first claim.

For the ranking claim, fix $i,j$ with $N_i>N_j$. We have
\[
Z_p\bigl(c_i(p)-c_j(p)\bigr)
=
(N_i-N_j)q_p+R_{ij}(q_p),
\]
where
\[
R_{ij}(q)
:=
\sum_{x\in S:\,d(x)\ge2}
\left(
\mathbf 1\{\tilde\ell_i(x)=0\}
-
\mathbf 1\{\tilde\ell_j(x)=0\}
\right)q^{d(x)}.
\]
If $0<q<1$, then $|R_{ij}(q)|\le |S|q^2$. Suppose $q_p<1/|S|$. Since
$S$ is nonempty, $1/|S|\le1$, so the bound on $R_{ij}$ applies at
$q=q_p$, and $N_i-N_j\ge1$ gives
\[
Z_p\bigl(c_i(p)-c_j(p)\bigr)
\ge
q_p-|S|q_p^2
=
q_p\bigl(1-|S|q_p\bigr)
>0.
\]
The condition $q_p<1/|S|$ is equivalent to
\[
p<\frac{1}{1+\kappa|S|}.
\]
Since $Z_p>0$, the ranking claim follows. The normalized scores
$\bar c_i(p)$ preserve the same ordering because they divide every
$c_i(p)$ by the common positive factor $\sum_{j=1}^T c_j(p)$.
\end{proof}

%% file: appendix/app_methodology.tex
\section{Additional Experimental Results }\label{app:results}

\subsection{Controlled Localization and Repair}
\label{app:controlled-results}

We use the candidate sets from the 20Q benchmark of \citet{NEURIPS2024_2b0e14ab}, spanning four domains: Animals (27 candidates), Food (26), Objects (37), and Places (21). For each domain we wrote a pool of yes/no questions and encoded each question as a binary predicate over the candidates, giving an attribute matrix with one row per candidate and one column per question, where entry $(x,q)$ is the true answer to question $q$ for candidate $x$. The matrix determines the truthful answers and the candidate-set updates and each target can be uniquely identified through the questions. 

Under the Natural question policy of \Cref{sec:results-controlled}, a language model uses the question text to select among the available predicates. In this section we report a series of robustness checks to our main results. In particular, we report localization and recovery results across different question-selection
policies, evaluate recovery with a larger candidate set, and examine sensitivity to the assumed update reliability and variants of the repair controller.

Throughout, evaluation is conditioned on an initial failure. For localization, Hit@1 is the fraction of scored histories whose top-ranked update is invalid. Average precision is computed for each history by averaging the precision at the ranks occupied by invalid updates. Equal scores are ordered by earlier turn. For repair, we measure the mean number of additional interactions required to recover: each question, verification and subsequent wrong guess costs one, while masking and the final correct guess cost zero. Costs exclude the preceding interaction and the initial rejection. 

We set the assumed reliability to $\rho_i=1-p$ for the per-turn error-injection experiments and $\rho_i=0.9$ for the single-injection experiments. The experiments use error weights $\kappa_i=1$.

\subsubsection{Localization Across Question Policies}
\label{app:policies}

\Cref{tab:loc-policy-natural,tab:loc-policy-voi,tab:loc-policy-random} report localization results under three question policies: Natural, VoI, and Random. The assistant asks predicates from the fixed pool. Under Natural, the language model picks the next question given the question history and the candidates that survive its current state; under VoI it picks the predicate with the largest expected information gain over its current candidate set; under Random it picks uniformly among unasked predicates that still split its candidate set.

In every case, the user answers truthfully from the attribute matrix. An admissible replacement predicate excludes the target while retaining at least one current candidate. When such replacements exist, we select one uniformly with probability $p$ and apply the intended update with probability $1-p$. If no admissible replacement exists, we apply the intended update. We evaluate unsuccessful histories, including those ending with a rejected guess, an empty candidate set or exhaustion of the question budget. Invalid updates are the turns whose applied constraint excludes the target; $K$ is their number and $T$ the conversation length at failure.

\begin{table}[htpb]
  \centering \small \setlength{\tabcolsep}{3pt}
  \resizebox{\linewidth}{!}{\input{tables/localization_policy_natural}}
  \caption{Localization under the Natural question policy (mean $\pm$ SE). Higher is better; bold marks the highest mean for each metric in each column.}
  \label{tab:loc-policy-natural}
\end{table}

\begin{table}[htpb]
  \centering \small \setlength{\tabcolsep}{3pt}
  \resizebox{\linewidth}{!}{\input{tables/localization_policy_voi}}
  \caption{Localization under the VoI question policy (mean $\pm$ SE). Higher is better; bold marks the highest mean for each metric in each column.}
  \label{tab:loc-policy-voi}
\end{table}

\begin{table}[htpb]
  \centering \small \setlength{\tabcolsep}{3pt}
  \resizebox{\linewidth}{!}{\input{tables/localization_policy_random}}
  \caption{Localization under the Random question policy (mean $\pm$ SE). Higher is better; bold marks the highest mean for each metric in each column.}
  \label{tab:loc-policy-random}
\end{table}

\subsubsection{Recovery Across Question Policies}
\label{app:recovery-policies}
\Cref{tab:repair-perfect-natural,tab:repair-perfect-random,tab:repair-perfect-voi} report repair results under the three question policies. The tables share the same repair methods and error-rate scenarios. All the baselines that support localization (i.e., all but Restart) are evaluated under the always-\texttt{VERIFY} policy. For \textsc{Loci} scores, as they support uncertainty quantification, we also evaluate the adaptive policy, together with the always-\texttt{MASK} policy. 

Questions, verifications and later wrong guesses cost one; masks and the correct guess cost zero. For each domain and error-rate scenario, we fit 26 threshold pairs separately for each adaptive variant using five target folds. Every repetition of a target stays in the same fold, and we select the lowest training mean cost.

\begin{table}[htpb]
  \centering\small\setlength{\tabcolsep}{3pt}
  \scalebox{0.9554}{\input{tables/repair_perfect_natural}}
  \caption{Recovery under the Natural question policy (mean $\pm$ SE). Lower is better; bold marks the lowest mean in each column. $^{*}$The \texttt{MASK} mean at $p=0.5$ excludes 12 episodes that reached the action limit.}
  \label{tab:repair-perfect-natural}
\end{table}

\begin{table}[htpb]
  \centering\small\setlength{\tabcolsep}{3pt}
  \scalebox{0.9554}{\input{tables/repair_perfect_voi}}
  \caption{Recovery under the VoI question policy (mean $\pm$ SE). Lower is better; bold marks the lowest mean in each column.}
  \label{tab:repair-perfect-voi}
\end{table}

\begin{table}[htpb]
  \centering\small\setlength{\tabcolsep}{3pt}
  \scalebox{0.9554}{\input{tables/repair_perfect_random}}
  \caption{Recovery under the Random question policy (mean $\pm$ SE). Lower is better; bold marks the lowest mean in each column.}
  \label{tab:repair-perfect-random}
\end{table}

\subsubsection{Larger Candidate Sets}
\label{app:repair-cub}

To evaluate repair with a larger candidate set, we use the 200 bird species in CUB-200-2011 \citep{wah2011caltech}. We represent each species by yes/no attributes, marking an attribute as present when it occurs in at least 50\% of that species' annotated examples. Removing uninformative and redundant attributes leaves 169 possible questions.  Assumed reliability is fixed at $\rho=1-p$, with unit error weights $\kappa_i=1$. 
 \Cref{tab:repair-cub} reports mean additional interactions conditional on initial failure. 

\begin{table}[htpb]
  \centering\small\setlength{\tabcolsep}{3pt}
  \scalebox{0.9554}{\input{tables/repair_cub_confirmation}}
  \caption{Recovery on CUB (mean $\pm$ SE). Lower is better; bold marks the lowest mean in each column.}
  \label{tab:repair-cub}
\end{table}

\subsection{Free-Form Language-Model Interactions}
\label{app:freeform-results}

We generate free-form conversations on the 20Q game from \citet{NEURIPS2024_2b0e14ab} with Llama 3.1 8B
playing both roles. The assistant writes its own yes/no questions and is not restricted to the
predicate pool; the simulated user answers each question for the hidden target, which belongs to the fixed candidate set for each domain.  Each conversation allows up to 20 questions, and a conversation counts as failed when the assistant does not identify the target within this budget.

\subsubsection{Localization Results}
\label{app:free-form}

Free-form conversations do not provide ground-truth labels for invalid updates, hence localization accuracy is hard to assess. For this, we use a separate LLM judge to construct reference labels for localization. These labels depend on the judge's accuracy and interpretation of ambiguous questions. This motivates the controlled evaluation with injected mistakes in \Cref{sec:results-controlled} and the recovery comparison under an always-\texttt{VERIFY} policy in \Cref{sec:free-form}.
Nevertheless, here we present some results to assess the localization quality using an LLM judge. 

We use Fable 5 to separately annotate the answers for each question, given the known target. We label a turn invalid when the recorded user answer contradicts the judge's unambiguous answer. If it labels a question as ambiguous, then we do not mark it as invalid. We score the 20 failed conversations with at least one judge-labeled invalid turn. \Cref{tab:freeform-localization} reports Hit@1 and AP against these reference labels.

\begin{table}[htpb]
  \centering
  \small
  \scalebox{0.9554}{%
  \begin{tabular}{lcc}
    \toprule
     & Hit@1 & AP \\
    \midrule
    First & $0.200 \pm 0.089$ & $0.353 \pm 0.037$ \\
    Last & $0.450 \pm 0.108$ & $0.582 \pm 0.052$ \\
    Random & $0.350 \pm 0.106$ & $0.446 \pm 0.038$ \\
    \midrule
    Counting & $0.200 \pm 0.089$ & $0.353 \pm 0.038$ \\
    \textsc{Loci} & $\mathbf{0.750 \pm 0.098}$ & $\mathbf{0.643 \pm 0.037}$ \\
    \bottomrule
  \end{tabular}}
  \caption{Localization on 20 failed conversations in the Animals domain using reference labels from an LLM judge (mean $\pm$ SE).}
  \label{tab:freeform-localization}
\end{table}

\subsubsection{LLM-Based Localization}
\label{app:llama-selector}

Here, we test whether the interacting language model can identify which earlier updates need repair. We compare \textsc{Loci} with Llama 3.1 8B to select the location of errors in free-form conversations. The LLM receives the recorded questions and answers, the candidate sets retained by each update, rejected guesses, and each update's repair status, but does not see the hidden target, reference error labels, or \textsc{Loci} scores. It ranks unresolved updates by how likely they are to have excluded the target. The ranking is reused during consecutive verification checks, and a new failure triggers a new ranking. Here, the methods use the same always-\texttt{VERIFY} controller to ensure fairness in the comparison. 

\begin{table}[htpb]
  \centering\small\setlength{\tabcolsep}{3pt}
  \scalebox{0.9554}{\input{tables/freeform_llama_selector}}
  \caption{Free-form recovery under different update selectors with always-\texttt{VERIFY}, measured by interaction score (mean $\pm$ SE). Lower is better; bold marks the lowest mean in each column.}
  \label{tab:freeform-llama-selector}
\end{table}

In \Cref{tab:freeform-llama-selector}, \textsc{Loci} has the lowest mean cost in every domain. On average it required $12.47$ interactions, compared with $21.19$ for Llama, a $41.1\%$ reduction. Most of the saving comes from fewer verification checks: Llama requires $18.47$ checks on average and \textsc{Loci} requires $10.96$. We note that, though this comparison evaluates localization through recovery cost, it also provides a proxy for error-location accuracy given the fixed policy. 

The bad localization performance of Llama may be due to the fact that the LLM selector shares some of the errors that produced the failed interaction. In fact, the same model constructs the candidate-set updates and later tries to diagnose them, so a mistaken belief or interpretation may persist during diagnosis. If the model failed during the conversation, in principle, it may also be unable to diagnose that failure. This motivates the external repair controller approach that we propose in this paper. In fact, we see \textsc{Loci} performing better. Moreover, we note that a stronger or different model may improve diagnosis, and providing it with \textsc{Loci} scores may combine the benefits of both approaches. These possibilities remain untested.

\subsubsection{Recovery Results}
\label{app:freeform-repair}

We follow the free-form setup in \Cref{sec:results-freeform}. \Cref{tab:recovery-summary} reports recovery costs, and \Cref{fig:controlled-repair-components} separates them into questions, verification checks, and rejected guesses. 
Verification accounts for 90\% of the always-\texttt{VERIFY} cost in free-form conversations, compared with 42\% in the controlled setting. Adaptive repair saves 7.17 checks on average while adding 3.34 questions and rejected guesses, yielding a lower total cost. Compared with always-\texttt{MASK}, it requires fewer questions and rejected guesses, although it incurs verification costs. 

\begin{figure}[htpb]
  \centering
  \includegraphics[width=\linewidth]{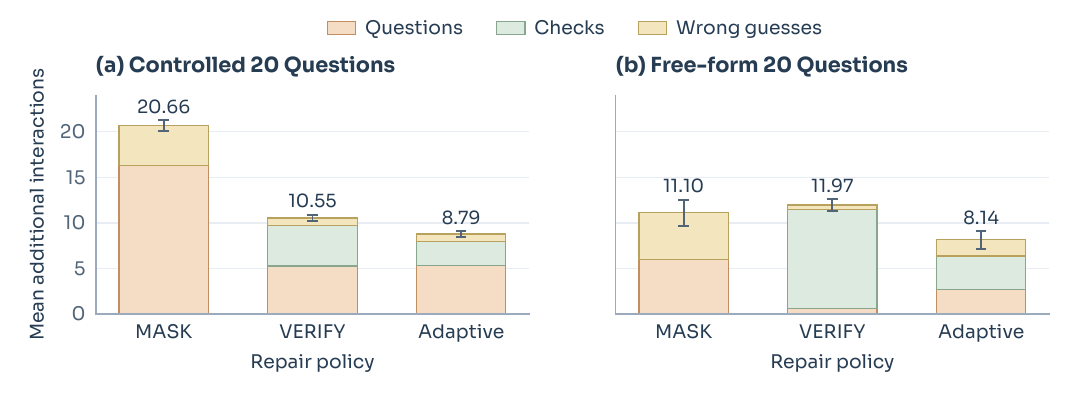}
  \caption{\textbf{Recovery cost under different \textsc{Loci} repair policies.} Stacked bars show additional questions, verification checks, and rejected guesses for always-\texttt{MASK}, always-\texttt{VERIFY}, and adaptive repair in (a) controlled 20Q and (b) free-form conversations. Means weight the four domains equally.}
  \label{fig:controlled-repair-components}
\end{figure}

For robustness, we repeat the recovery evaluation with Qwen3-8B (with thinking disabled). The user retains the conversation history. We use the same four domains, hyperparameters, interaction costs, and budget of 100. Each model generates its own failure cohort, so we compare methods within each model. \Cref{tab:freeform-repair-qwen} shows the results. \textsc{Loci} with always-\texttt{VERIFY} reduces mean recovery cost by 34.3\% relative to Counting with the same action. Adaptive repair has the lowest mean cost in each domain and reduces the equal-domain mean by 51.1\% relative to Counting. All always-\texttt{VERIFY} methods and adaptive \textsc{Loci} recover all 87 failures; always-\texttt{MASK} recovers 86 and Restart recovers 69. Runs that do not recover within the interaction budget are assigned a cost of 100.

\begin{table}[htpb]
\centering\small
\setlength{\tabcolsep}{3pt}
\scalebox{0.9554}{%
\begin{tabular}{lrrrrr}
\toprule
Method & Average & Animals & Food & Objects & Places \\
\midrule
First + \texttt{VERIFY} & $21.22 \pm 0.84$ & $21.39 \pm 1.57$ & $20.45 \pm 1.66$ & $25.30 \pm 1.50$ & $17.71 \pm 1.88$ \\
Last + \texttt{VERIFY} & $20.91 \pm 1.13$ & $19.00 \pm 1.48$ & $18.91 \pm 1.36$ & $24.52 \pm 1.69$ & $21.21 \pm 3.56$ \\
Random + \texttt{VERIFY} & $22.17 \pm 1.03$ & $21.22 \pm 1.36$ & $19.95 \pm 1.75$ & $25.94 \pm 1.64$ & $21.57 \pm 2.98$ \\
Counting + \texttt{VERIFY} & $20.18 \pm 1.03$ & $20.61 \pm 1.45$ & $19.23 \pm 1.93$ & $24.52 \pm 1.47$ & $16.36 \pm 2.92$ \\
Restart & $47.57 \pm 3.36$ & $43.11 \pm 7.03$ & $43.18 \pm 6.55$ & $70.79 \pm 6.46$ & $33.21 \pm 7.04$ \\
\textsc{Loci} + \texttt{VERIFY} & $13.26 \pm 0.87$ & $12.78 \pm 1.48$ & $9.86 \pm 1.21$ & $20.12 \pm 2.08$ & $10.29 \pm 1.92$ \\
\textsc{Loci} + \texttt{MASK} & $16.23 \pm 1.75$ & $15.39 \pm 3.82$ & $10.14 \pm 2.38$ & $28.91 \pm 4.56$ & $10.50 \pm 3.14$ \\
\textsc{Loci} adaptive & $\mathbf{9.86 \pm 1.03}$ & $\mathbf{6.94 \pm 1.33}$ & $\mathbf{6.27 \pm 1.08}$ & $\mathbf{18.15 \pm 3.07}$ & $\mathbf{8.07 \pm 2.14}$ \\
\bottomrule
\end{tabular}}
\caption{Free-form recovery with Qwen3-8B, measured by interaction score (mean $\pm$ SE). Lower is better; bold marks the lowest mean in each column.}
\label{tab:freeform-repair-qwen}
\end{table}

\subsection{Task-Oriented Dialogue}
\label{app:task-dialogue}

\subsubsection{Setup}
\label{app:multiwoz-setup}

We use the hotel, restaurant, and attraction search requirements in MultiWOZ 2.4. Each slot is a named requirement, such as hotel area or price range, and each candidate intent $x\in\mathcal X$ assigns values to the specified slots; unspecified or unrestricted slots do not impose any constraints. The set of specified slots is known. We construct $\mathcal X$ by combining configurations that are supported by the database across domains. Each exchange contributes one constraint containing its newly specified values, e.g. ``a cheap hotel in the east'' gives \texttt{hotel.price} $=$ \texttt{cheap} and \texttt{hotel.area} $=$ \texttt{east}. 
If a requirement is revised or removed, we truncate the dialogue before that turn to preserve a fixed target throughout the interaction. We retain dialogues with at least two turns that introduce requirements, so localization involves a choice among updates. These restrictions leave us with 346 test dialogues.

Each original update is corrupted independently with probability $p$. For a corrupted update, we choose one of its slots uniformly and replace its value with a uniformly chosen alternative from that slot's database vocabulary. Let $k_i$ be the number of slots in update $i$ and $m_s$ the vocabulary size of slot $s$. For a candidate $x$ incompatible with the recorded update, this mechanism gives the error weight
\[
q_i(x)=
\begin{cases}
1/[k_i(m_s-1)], & \text{if only slot }s\text{ differs},\\
0, & \text{if multiple slots differ}.
\end{cases}
\]
We use assumed reliabilities $\rho_i=1-p$ for these original updates and a uniform prior over $\mathcal X$, conditioning on rejection of the recorded requirement configuration. For the rest, we follow the controlled evaluation of \Cref{sec:results-controlled}, with the difference that errors change slot values rather than the binary answers.

\subsubsection{Localization Results}
\label{app:multiwoz}

 We evaluate localization before repair using Hit@1 and AP on the same failed histories used for recovery. \Cref{tab:multiwoz-localization} reports the results. \textsc{Loci} consistently outperforms the baselines across the different error-rate scenarios. 

\begin{table}[htpb]
\centering
\caption{Localization on MultiWOZ (mean $\pm$ SE). Higher is better; bold marks the highest mean for each metric in each column. The Average column weights the five error rates equally.}
\label{tab:multiwoz-localization}
\small\setlength{\tabcolsep}{3pt}
\scalebox{0.9554}{\input{tables/multiwoz_localization}}
\end{table}

\subsubsection{Recovery Results}
\label{app:multiwoz-repair}

Recovery follows the repair policies and interaction costs described in \Cref{sec:repair}. New slot answers during recovery are reliable and restrict the candidate set.  \Cref{tab:multiwoz-recovery} reports recovery costs on the same histories used for localization. All methods recover every history within the interaction budget.

\begin{table}[htpb]
\centering\small
\setlength{\tabcolsep}{3pt}
{\scalebox{0.9554}{\input{tables/multiwoz_recovery}}}
\caption{Recovery on MultiWOZ, measured by interaction score (mean $\pm$ SE). Lower is better; bold marks the lowest mean in each column. The Average column weights the five error rates equally.}
\label{tab:multiwoz-recovery}
\end{table}

\textsc{Loci} beats the baselines across all the error rates, with always-\texttt{MASK} achieving the lowest mean cost. \Cref{tab:multiwoz-recovery-components} shows that adaptive repair saves questions and wrong guesses, but these savings are smaller than its verification costs. This is because reliable slot answers may make discarded information relatively cheap to recover here. Always-\texttt{MASK} also resumes questioning once candidates remain, while adaptive repair can continue checking updates. This suggests that the value of verification depends on both corruption confidence and the cost of recovering information through the subsequent turns of the conversation.

\begin{table}[htbp]
\centering\small
\setlength{\tabcolsep}{4pt}
\scalebox{0.9554}{\input{tables/multiwoz_recovery_components}}
\caption{Recovery-cost components for \textsc{Loci} (mean $\pm$ source-dialogue-cluster bootstrap SE). Means weight error rates equally.}
\label{tab:multiwoz-recovery-components}
\end{table}

\subsection{Candidate Intent Spaces}
\label{app:candidate-intents}

\textsc{Loci} attributes errors using a belief over possible user intents and their compatibility with each recorded update. In our experiments, we use finite candidate spaces containing the target, which permit exact computation of the scores under the specified model. For tasks where the candidate space is unknown, retrieval or generation could construct a working set during the interaction. \citet{NEURIPS2024_2b0e14ab} provide an approach to generalize this: an LLM proposes a finite candidate set using the conversation history and refreshes it after each interaction. An analogous extension for \textsc{Loci} would propose alternative intents using the original exchanges and rejection feedback, evaluate the recorded constraints on these candidates, and assign weights approximating the intent posterior. The same attribution rule would then apply to this representation. Candidate generation must allow alternatives that contradict the assistant's recorded interpretations, since these alternatives explain how an earlier update could have excluded the target.

Localization can also precede identification of the exact intent. By \Cref{eq:corruption}, candidates with the same pattern of compatibility across recorded updates can have their posterior probabilities combined without changing any corruption score. For example, several flights departing after 6 p.m.\ can all provide evidence against an incorrectly recorded requirement to depart before 6 p.m., even while the intended flight remains uncertain. A candidate representation that preserves the posterior mass assigned to each compatibility pattern therefore preserves the localization scores. Constructing such a representation requires accounting for the weights of the alternatives it represents.
These extensions could apply to interactions where the user has a fixed target but the candidate space is not predefined. The quality of the resulting scores would depend on candidate coverage, weighting, and accurate constraint evaluation. We leave these extensions for future work.

\subsection{Ablations}
\subsubsection{Sensitivity to Assumed Reliability}
\label{app:reliability-sensitivity}

To assess sensitivity to the choice of $\rho$ in \Cref{eq:update-likelihood}, we vary the common assumed reliability over $\{0.5,0.7,0.9,0.95\}$ while holding the initial histories fixed. We test its effect on controlled localization and on free-form recovery with fixed repair thresholds. Results are reported in \Cref{tab:reliability-localization}, which also includes Counting as a reference. 

\begin{table}[htpb]
  \centering\small\setlength{\tabcolsep}{3pt}
  \scalebox{0.9554}{\input{tables/reliability_localization}}
  \caption{Controlled localization under different assumed reliabilities (mean $\pm$ SE). Higher is better. Each column uses the same failed histories for every row.}
  \label{tab:reliability-localization}
\end{table}

Overall, we observe that localization depends on the prior reliability in the controlled setup. For example, at $p=0.1$, reducing $\rho$ from $0.9$ to $0.5$ lowers AP from $0.694$ to $0.609$, while at $p=0.5$, increasing $\rho$ from $0.5$ to $0.9$ lowers Hit@1 from $0.744$ to $0.643$ and AP from $0.808$ to $0.760$. No fixed reliability performs best throughout the grid.

The belief weights explain why high reliability can hurt when many updates are invalid. With unit error weights $\kappa_i=1$, each additional violated constraint multiplies a candidate's relative weight by $(1-\rho)/\rho$. At $\rho=0.9$, two additional violations incur an $81$-fold penalty. A true target excluded by several corrupted updates can therefore receive little posterior mass, reducing its contribution to their corruption scores. Lower reliability weakens this penalty, but also weakens the information supplied by valid updates.

For recovery, we reuse the failed Llama 3.1 8B conversations and recorded candidate classifications from \Cref{app:freeform-repair}. We rerun \textsc{Loci} + \texttt{VERIFY} and adaptive repair at each reliability. The adaptive thresholds selected on other target folds at $\rho=0.9$ remain fixed throughout. Results are in \Cref{tab:reliability-recovery}. 
Mean always-\texttt{VERIFY} cost ranges from $12.17$ to $12.32$ interactions, and adaptive cost from $8.36$ to $9.79$. Every paired 95\% bootstrap interval for a cost change relative
to $\rho=0.9$ includes zero.

\begin{table}[htpb]
  \centering\small\setlength{\tabcolsep}{8pt}
  \scalebox{0.9554}{\input{tables/reliability_recovery}}
  \caption{Free-form recovery under different assumed reliabilities, measured by additional interactions (mean $\pm$ target-cluster bootstrap SE). Means weight domains equally.}
  \label{tab:reliability-recovery}
\end{table}

\subsubsection{Controller Ablations}
\label{app:controller-ablations}

\Cref{tab:natural-controller-ablation} compares evidence retention and repair stopping rules on the same Natural-policy histories. Adaptive retain keeps masked updates as evidence for posterior inference; adaptive forget discards them. Feasibility-stop retains this evidence and resumes questioning when masking leaves a nonempty candidate set or verification confirms an update. Each variant uses thresholds fitted separately on the same target folds. Mean recovery costs are similar across the three variants at each error rate.

\begin{table}[htpb]
  \centering\small\setlength{\tabcolsep}{5pt}
\scalebox{0.9554}{%
\begin{tabular}{lcccc}
\toprule
Controller & $p=0.1$ & $p=0.2$ & $p=0.3$ & $p=0.4$ \\
\midrule
Adaptive retain & $6.31 \pm 0.34$ & $7.66 \pm 0.34$ & $9.50 \pm 0.37$ & $11.69 \pm 0.47$ \\
Adaptive forget & $6.35 \pm 0.34$ & $7.75 \pm 0.33$ & $9.51 \pm 0.37$ & $11.64 \pm 0.47$ \\
Feasibility-stop & $6.30 \pm 0.34$ & $7.61 \pm 0.34$ & $9.50 \pm 0.37$ & $11.61 \pm 0.47$ \\
\bottomrule
\end{tabular}}
  \caption{Recovery under the Natural question policy with different evidence retention and stopping rules (mean $\pm$ SE).}
  \label{tab:natural-controller-ablation}
\end{table}

\subsubsection{Effect of User Conversation History}
\label{app:user-history}

\Cref{tab:freeform-repair-nohist} compares recovery with and without user history. With history, the simulated user sees the conversation; without it, the user sees only the target and current question. The assistant retains its unmasked conversation in both settings. \textsc{Loci} adaptive has the lowest mean recovery cost in all four domains with user history and in two of four without it.

\begin{table}[htpb]
\centering\small
\setlength{\tabcolsep}{5pt}
\scalebox{0.9554}{%
\begin{tabular}{lrrrr}
\toprule
Method & Animals & Food & Objects & Places \\
\midrule
\multicolumn{5}{c}{User with conversation history} \\
\midrule
First + \texttt{VERIFY} & $22.29 \pm 1.79$ & $20.62 \pm 1.38$ & $25.33 \pm 1.76$ & $20.27 \pm 1.77$ \\
Last + \texttt{VERIFY} & $19.71 \pm 1.28$ & $18.57 \pm 1.12$ & $23.67 \pm 1.62$ & $18.73 \pm 1.57$ \\
Random + \texttt{VERIFY} & $20.67 \pm 0.94$ & $19.95 \pm 0.94$ & $26.00 \pm 1.97$ & $20.91 \pm 1.50$ \\
Counting + \texttt{VERIFY} & $21.52 \pm 2.34$ & $18.62 \pm 0.90$ & $24.75 \pm 2.02$ & $19.00 \pm 1.56$ \\
Restart & $43.67 \pm 5.90$ & $46.71 \pm 6.88$ & $56.88 \pm 7.71$ & $20.09 \pm 6.47$ \\
\textsc{Loci} + \texttt{VERIFY} & $12.05 \pm 1.34$ & $11.19 \pm 1.03$ & $15.29 \pm 1.41$ & $9.36 \pm 1.22$ \\
\textsc{Loci} + \texttt{MASK} & $12.52 \pm 3.32$ & $10.48 \pm 2.64$ & $15.42 \pm 3.08$ & $6.00 \pm 2.23$ \\
\textsc{Loci} adaptive & $\mathbf{10.62 \pm 2.50}$ & $\mathbf{7.10 \pm 1.40}$ & $\mathbf{10.83 \pm 2.13}$ & $\mathbf{4.00 \pm 1.48}$ \\
\midrule
\multicolumn{5}{c}{User without conversation history} \\
\midrule
First + \texttt{VERIFY} & $17.52 \pm 0.96$ & $16.86 \pm 1.83$ & $18.68 \pm 1.06$ & $34.00 \pm 10.49$ \\
Last + \texttt{VERIFY} & $16.00 \pm 0.80$ & $18.43 \pm 1.31$ & $18.95 \pm 1.12$ & $36.00 \pm 10.00$ \\
Random + \texttt{VERIFY} & $17.29 \pm 0.85$ & $18.93 \pm 1.17$ & $18.11 \pm 0.96$ & $34.20 \pm 10.36$ \\
Counting + \texttt{VERIFY} & $15.38 \pm 1.13$ & $15.07 \pm 1.51$ & $14.68 \pm 1.14$ & $29.10 \pm 10.23$ \\
Restart & $47.76 \pm 7.58$ & $43.29 \pm 7.06$ & $48.42 \pm 9.25$ & $33.10 \pm 10.46$ \\
\textsc{Loci} + \texttt{VERIFY} & $7.10 \pm 0.79$ & $7.29 \pm 1.22$ & $\mathbf{7.42 \pm 1.08}$ & $19.80 \pm 9.49$ \\
\textsc{Loci} + \texttt{MASK} & $7.14 \pm 2.20$ & $4.36 \pm 1.77$ & $15.58 \pm 4.08$ & $\mathbf{17.30 \pm 9.30}$ \\
\textsc{Loci} adaptive & $\mathbf{4.38 \pm 1.54}$ & $\mathbf{3.93 \pm 1.30}$ & $8.21 \pm 1.61$ & $17.90 \pm 9.84$ \\
\bottomrule
\end{tabular}}
\caption{Recovery on LLM-generated conversations with and without user history, measured by interaction score (mean $\pm$ SE). Lower is better; bold marks the lowest mean in each column within each history setting.}
\label{tab:freeform-repair-nohist}
\end{table}

%% file: tables/localization_policy_natural.tex
\begin{tabular}{lcccccc}
\toprule
 & 1 error & $p=0.1$ & $p=0.2$ & $p=0.3$ & $p=0.4$ & $p=0.5$ \\
\midrule
Scored failures & 4,275 & 565 & 822 & 970 & 1,038 & 1,066 \\
Mean $T$ & 4.78 & 6.60 & 6.28 & 5.90 & 5.36 & 5.24 \\
Mean $K$ & 1.00 & 1.26 & 1.63 & 1.97 & 2.26 & 2.70 \\
\midrule
\multicolumn{7}{l}{Hit@1} \\
\midrule
First & $0.313 \pm 0.016$ & $0.225 \pm 0.019$ & $0.273 \pm 0.015$ & $0.344 \pm 0.017$ & $0.433 \pm 0.018$ & $0.526 \pm 0.017$ \\
Last & $0.352 \pm 0.011$ & $0.360 \pm 0.026$ & $0.353 \pm 0.022$ & $0.376 \pm 0.020$ & $0.423 \pm 0.019$ & $0.440 \pm 0.018$ \\
Random & $0.296 \pm 0.009$ & $0.256 \pm 0.020$ & $0.310 \pm 0.016$ & $0.392 \pm 0.015$ & $0.474 \pm 0.015$ & $0.549 \pm 0.015$ \\
\cmidrule(lr){1-7}
Counting & $\mathbf{0.737 \pm 0.020}$ & $0.500 \pm 0.023$ & $\mathbf{0.530 \pm 0.022}$ & $0.564 \pm 0.021$ & $0.628 \pm 0.022$ & $0.628 \pm 0.019$ \\
\textsc{Loci} & $0.719 \pm 0.019$ & $\mathbf{0.509 \pm 0.024}$ & $0.528 \pm 0.021$ & $\mathbf{0.586 \pm 0.020}$ & $\mathbf{0.674 \pm 0.021}$ & $\mathbf{0.744 \pm 0.021}$ \\
\midrule
\multicolumn{7}{l}{AP} \\
\midrule
First & $0.545 \pm 0.011$ & $0.443 \pm 0.016$ & $0.494 \pm 0.011$ & $0.558 \pm 0.011$ & $0.633 \pm 0.010$ & $0.703 \pm 0.009$ \\
Last & $0.583 \pm 0.007$ & $0.575 \pm 0.017$ & $0.585 \pm 0.013$ & $0.621 \pm 0.012$ & $0.660 \pm 0.010$ & $0.693 \pm 0.009$ \\
Random & $0.538 \pm 0.008$ & $0.486 \pm 0.014$ & $0.532 \pm 0.010$ & $0.586 \pm 0.009$ & $0.649 \pm 0.009$ & $0.705 \pm 0.008$ \\
\cmidrule(lr){1-7}
Counting & $\mathbf{0.849 \pm 0.012}$ & $0.682 \pm 0.015$ & $\mathbf{0.686 \pm 0.014}$ & $\mathbf{0.702 \pm 0.013}$ & $0.748 \pm 0.012$ & $0.760 \pm 0.010$ \\
\textsc{Loci} & $0.837 \pm 0.012$ & $\mathbf{0.693 \pm 0.016}$ & $0.685 \pm 0.014$ & $0.701 \pm 0.014$ & $\mathbf{0.755 \pm 0.013}$ & $\mathbf{0.808 \pm 0.011}$ \\
\bottomrule
\end{tabular}

%% file: tables/localization_policy_voi.tex
\begin{tabular}{lcccccc}
\toprule
 & 1 error & $p=0.1$ & $p=0.2$ & $p=0.3$ & $p=0.4$ & $p=0.5$ \\
\midrule
Scored failures & 2,527 & 2,013 & 3,351 & 4,274 & 4,824 & 5,178 \\
Mean $T$ & 3.77 & 4.70 & 4.69 & 4.69 & 4.69 & 4.71 \\
Mean $K$ & 1.00 & 1.21 & 1.46 & 1.77 & 2.11 & 2.51 \\
\midrule
\multicolumn{7}{l}{Hit@1} \\
\midrule
First & $0.396 \pm 0.020$ & $0.283 \pm 0.012$ & $0.328 \pm 0.009$ & $0.396 \pm 0.009$ & $0.468 \pm 0.009$ & $0.539 \pm 0.008$ \\
Last & $0.345 \pm 0.011$ & $0.383 \pm 0.016$ & $0.420 \pm 0.014$ & $0.472 \pm 0.012$ & $0.520 \pm 0.010$ & $0.601 \pm 0.010$ \\
Random & $0.377 \pm 0.010$ & $0.346 \pm 0.012$ & $0.368 \pm 0.009$ & $0.427 \pm 0.008$ & $0.492 \pm 0.007$ & $0.572 \pm 0.008$ \\
\cmidrule(lr){1-7}
Counting & $\mathbf{0.783 \pm 0.015}$ & $0.555 \pm 0.023$ & $0.571 \pm 0.021$ & $0.588 \pm 0.019$ & $0.608 \pm 0.017$ & $0.636 \pm 0.017$ \\
\textsc{Loci} & $0.773 \pm 0.016$ & $\mathbf{0.569 \pm 0.023}$ & $\mathbf{0.593 \pm 0.022}$ & $\mathbf{0.626 \pm 0.020}$ & $\mathbf{0.671 \pm 0.018}$ & $\mathbf{0.744 \pm 0.017}$ \\
\midrule
\multicolumn{7}{l}{AP} \\
\midrule
First & $0.636 \pm 0.012$ & $0.530 \pm 0.009$ & $0.563 \pm 0.007$ & $0.608 \pm 0.006$ & $0.657 \pm 0.006$ & $0.704 \pm 0.005$ \\
Last & $0.572 \pm 0.008$ & $0.587 \pm 0.011$ & $0.617 \pm 0.010$ & $0.656 \pm 0.008$ & $0.693 \pm 0.007$ & $0.743 \pm 0.006$ \\
Random & $0.608 \pm 0.008$ & $0.567 \pm 0.009$ & $0.586 \pm 0.007$ & $0.626 \pm 0.006$ & $0.668 \pm 0.005$ & $0.721 \pm 0.005$ \\
\cmidrule(lr){1-7}
Counting & $\mathbf{0.869 \pm 0.010}$ & $\mathbf{0.707 \pm 0.015}$ & $\mathbf{0.706 \pm 0.014}$ & $0.712 \pm 0.013$ & $0.728 \pm 0.011$ & $0.752 \pm 0.010$ \\
\textsc{Loci} & $0.855 \pm 0.011$ & $0.702 \pm 0.017$ & $0.703 \pm 0.016$ & $\mathbf{0.720 \pm 0.014}$ & $\mathbf{0.757 \pm 0.012}$ & $\mathbf{0.819 \pm 0.010}$ \\
\bottomrule
\end{tabular}

%% file: tables/localization_policy_random.tex
\begin{tabular}{lcccccc}
\toprule
 & 1 error & $p=0.1$ & $p=0.2$ & $p=0.3$ & $p=0.4$ & $p=0.5$ \\
\midrule
Scored failures & 3,416 & 2,616 & 3,915 & 4,665 & 5,038 & 5,267 \\
Mean $T$ & 6.94 & 6.87 & 6.41 & 6.04 & 5.78 & 5.62 \\
Mean $K$ & 1.00 & 1.32 & 1.67 & 2.05 & 2.47 & 2.89 \\
\midrule
\multicolumn{7}{l}{Hit@1} \\
\midrule
First & $0.221 \pm 0.018$ & $0.221 \pm 0.008$ & $0.282 \pm 0.007$ & $0.363 \pm 0.008$ & $0.448 \pm 0.007$ & $0.530 \pm 0.007$ \\
Last & $0.229 \pm 0.010$ & $0.295 \pm 0.009$ & $0.340 \pm 0.009$ & $0.401 \pm 0.009$ & $0.474 \pm 0.010$ & $0.551 \pm 0.011$ \\
Random & $0.228 \pm 0.012$ & $0.249 \pm 0.009$ & $0.311 \pm 0.007$ & $0.388 \pm 0.008$ & $0.473 \pm 0.009$ & $0.559 \pm 0.008$ \\
\cmidrule(lr){1-7}
Counting & $\mathbf{0.697 \pm 0.016}$ & $\mathbf{0.538 \pm 0.014}$ & $0.549 \pm 0.013$ & $0.562 \pm 0.012$ & $0.581 \pm 0.013$ & $0.605 \pm 0.014$ \\
\textsc{Loci} & $0.646 \pm 0.016$ & $0.527 \pm 0.016$ & $\mathbf{0.554 \pm 0.014}$ & $\mathbf{0.593 \pm 0.014}$ & $\mathbf{0.652 \pm 0.013}$ & $\mathbf{0.768 \pm 0.014}$ \\
\midrule
\multicolumn{7}{l}{AP} \\
\midrule
First & $0.456 \pm 0.014$ & $0.442 \pm 0.006$ & $0.496 \pm 0.005$ & $0.559 \pm 0.005$ & $0.623 \pm 0.004$ & $0.683 \pm 0.004$ \\
Last & $0.458 \pm 0.009$ & $0.510 \pm 0.007$ & $0.549 \pm 0.006$ & $0.598 \pm 0.005$ & $0.653 \pm 0.006$ & $0.708 \pm 0.006$ \\
Random & $0.460 \pm 0.010$ & $0.469 \pm 0.006$ & $0.520 \pm 0.005$ & $0.576 \pm 0.005$ & $0.640 \pm 0.005$ & $0.696 \pm 0.004$ \\
\cmidrule(lr){1-7}
Counting & $\mathbf{0.827 \pm 0.009}$ & $\mathbf{0.693 \pm 0.009}$ & $0.677 \pm 0.008$ & $0.673 \pm 0.008$ & $0.686 \pm 0.008$ & $0.709 \pm 0.007$ \\
\textsc{Loci} & $0.789 \pm 0.010$ & $0.686 \pm 0.010$ & $\mathbf{0.685 \pm 0.009}$ & $\mathbf{0.705 \pm 0.009}$ & $\mathbf{0.750 \pm 0.009}$ & $\mathbf{0.827 \pm 0.008}$ \\
\bottomrule
\end{tabular}

%% file: tables/repair_perfect_natural.tex
\begin{tabular}{lccccc}
\toprule
 & $p=0.1$ & $p=0.2$ & $p=0.3$ & $p=0.4$ & $p=0.5$ \\
\midrule
First + \texttt{VERIFY} & $10.12 \pm 0.48$ & $12.28 \pm 0.49$ & $15.48 \pm 0.51$ & $18.29 \pm 0.68$ & $22.11 \pm 0.72$ \\
Last + \texttt{VERIFY} & $8.51 \pm 0.38$ & $10.69 \pm 0.41$ & $13.42 \pm 0.44$ & $16.56 \pm 0.61$ & $20.04 \pm 0.70$ \\
Random + \texttt{VERIFY} & $9.22 \pm 0.41$ & $11.48 \pm 0.45$ & $14.64 \pm 0.49$ & $17.56 \pm 0.65$ & $21.17 \pm 0.74$ \\
Counting + \texttt{VERIFY} & $7.36 \pm 0.36$ & $10.01 \pm 0.42$ & $13.25 \pm 0.47$ & $16.18 \pm 0.62$ & $20.60 \pm 0.69$ \\
Restart & $16.36 \pm 0.93$ & $31.87 \pm 2.43$ & $57.54 \pm 3.98$ & $91.00 \pm 5.37$ & $122.01 \pm 6.07$ \\
\midrule
\textsc{Loci} + \texttt{VERIFY} & $6.98 \pm 0.34$ & $9.06 \pm 0.35$ & $11.64 \pm 0.40$ & $14.49 \pm 0.53$ & $18.03 \pm 0.57$ \\
\textsc{Loci} + \texttt{MASK} & $7.36 \pm 0.40$ & $12.19 \pm 0.54$ & $23.02 \pm 0.83$ & $40.07 \pm 1.41$ & $68.26 \pm 2.29^{*}$ \\
\textsc{Loci} adaptive & $\mathbf{6.31 \pm 0.34}$ & $\mathbf{7.66 \pm 0.34}$ & $\mathbf{9.50 \pm 0.37}$ & $\mathbf{11.69 \pm 0.47}$ & $\mathbf{14.77 \pm 0.52}$ \\
\bottomrule
\end{tabular}

%% file: tables/repair_perfect_voi.tex
\begin{tabular}{lccccc}
\toprule
 & $p=0.1$ & $p=0.2$ & $p=0.3$ & $p=0.4$ & $p=0.5$ \\
\midrule
First + \texttt{VERIFY} & $5.72 \pm 0.23$ & $7.65 \pm 0.25$ & $9.68 \pm 0.29$ & $12.21 \pm 0.35$ & $16.03 \pm 0.44$ \\
Last + \texttt{VERIFY} & $5.77 \pm 0.25$ & $7.27 \pm 0.26$ & $9.33 \pm 0.31$ & $11.84 \pm 0.37$ & $15.66 \pm 0.45$ \\
Random + \texttt{VERIFY} & $5.75 \pm 0.23$ & $7.35 \pm 0.23$ & $9.35 \pm 0.28$ & $11.91 \pm 0.36$ & $15.69 \pm 0.45$ \\
Counting + \texttt{VERIFY} & $5.07 \pm 0.23$ & $6.98 \pm 0.26$ & $9.40 \pm 0.33$ & $12.22 \pm 0.39$ & $16.31 \pm 0.49$ \\
Restart & $8.34 \pm 0.47$ & $15.35 \pm 0.85$ & $24.17 \pm 1.01$ & $35.85 \pm 1.35$ & $48.17 \pm 1.38$ \\
\midrule
\textsc{Loci} + \texttt{VERIFY} & $5.14 \pm 0.23$ & $6.95 \pm 0.25$ & $9.23 \pm 0.32$ & $11.89 \pm 0.37$ & $15.44 \pm 0.46$ \\
\textsc{Loci} + \texttt{MASK} & $8.38 \pm 0.60$ & $15.95 \pm 0.88$ & $23.97 \pm 1.04$ & $29.77 \pm 0.98$ & $30.42 \pm 0.90$ \\
\textsc{Loci} adaptive & $\mathbf{4.53 \pm 0.21}$ & $\mathbf{6.19 \pm 0.23}$ & $\mathbf{8.08 \pm 0.28}$ & $\mathbf{10.31 \pm 0.33}$ & $\mathbf{14.32 \pm 0.42}$ \\
\bottomrule
\end{tabular}

%% file: tables/repair_perfect_random.tex
\begin{tabular}{lccccc}
\toprule
 & $p=0.1$ & $p=0.2$ & $p=0.3$ & $p=0.4$ & $p=0.5$ \\
\midrule
First + \texttt{VERIFY} & $7.67 \pm 0.30$ & $9.66 \pm 0.35$ & $12.49 \pm 0.38$ & $15.83 \pm 0.48$ & $19.88 \pm 0.51$ \\
Last + \texttt{VERIFY} & $7.29 \pm 0.29$ & $9.23 \pm 0.35$ & $12.30 \pm 0.39$ & $15.51 \pm 0.43$ & $19.51 \pm 0.47$ \\
Random + \texttt{VERIFY} & $7.58 \pm 0.29$ & $9.44 \pm 0.36$ & $12.49 \pm 0.38$ & $15.50 \pm 0.44$ & $19.82 \pm 0.47$ \\
Counting + \texttt{VERIFY} & $5.88 \pm 0.25$ & $8.32 \pm 0.34$ & $11.95 \pm 0.41$ & $15.77 \pm 0.46$ & $19.83 \pm 0.49$ \\
Restart & $12.46 \pm 0.66$ & $21.84 \pm 1.14$ & $34.72 \pm 1.70$ & $50.47 \pm 1.87$ & $60.39 \pm 1.62$ \\
\midrule
\textsc{Loci} + \texttt{VERIFY} & $5.74 \pm 0.25$ & $8.03 \pm 0.32$ & $11.24 \pm 0.36$ & $14.80 \pm 0.43$ & $18.13 \pm 0.45$ \\
\textsc{Loci} + \texttt{MASK} & $9.12 \pm 0.52$ & $16.81 \pm 0.67$ & $24.81 \pm 0.71$ & $32.75 \pm 0.69$ & $32.85 \pm 0.84$ \\
\textsc{Loci} adaptive & $\mathbf{5.38 \pm 0.24}$ & $\mathbf{7.22 \pm 0.31}$ & $\mathbf{9.84 \pm 0.34}$ & $\mathbf{12.90 \pm 0.42}$ & $\mathbf{16.99 \pm 0.46}$ \\
\bottomrule
\end{tabular}

%% file: tables/repair_cub_confirmation.tex
\begin{tabular}{lrrrrr}
\toprule
Method & $p=0.1$ & $p=0.2$ & $p=0.3$ & $p=0.4$ & $p=0.5$ \\
\midrule
First + \texttt{VERIFY} & $14.71 \pm 0.39$ & $19.92 \pm 0.45$ & $25.73 \pm 0.57$ & $32.17 \pm 0.67$ & $41.09 \pm 0.77$ \\
Last + \texttt{VERIFY} & $13.72 \pm 0.37$ & $19.00 \pm 0.43$ & $24.66 \pm 0.56$ & $31.20 \pm 0.63$ & $40.48 \pm 0.74$ \\
Random + \texttt{VERIFY} & $14.42 \pm 0.38$ & $19.46 \pm 0.42$ & $25.28 \pm 0.56$ & $31.51 \pm 0.62$ & $40.86 \pm 0.76$ \\
Counting + \texttt{VERIFY} & $9.72 \pm 0.30$ & $15.92 \pm 0.41$ & $22.78 \pm 0.55$ & $30.30 \pm 0.65$ & $40.41 \pm 0.79$ \\
Restart & $46.62 \pm 2.33$ & $128.69 \pm 6.56$ & $337.77 \pm 15.35$ & $622.74 \pm 22.75$ & $1005.93 \pm 28.17$ \\
\midrule
\textsc{Loci} + \texttt{VERIFY} & $9.45 \pm 0.27$ & $14.96 \pm 0.36$ & $21.16 \pm 0.50$ & $27.71 \pm 0.57$ & $36.01 \pm 0.66$ \\
\textsc{Loci} + \texttt{MASK} & $21.18 \pm 0.90$ & $65.00 \pm 2.62$ & $223.92 \pm 6.56$ & $401.02 \pm 3.77$ & $373.95 \pm 7.54$ \\
\textsc{Loci} adaptive & $\mathbf{9.11 \pm 0.27}$ & $\mathbf{14.14 \pm 0.35}$ & $\mathbf{19.29 \pm 0.47}$ & $\mathbf{25.37 \pm 0.56}$ & $\mathbf{34.68 \pm 0.68}$ \\
\bottomrule
\end{tabular}

%% file: tables/freeform_llama_selector.tex
\begin{tabular}{lccccc}
\toprule
Selector & Average & Animals & Food & Objects & Places \\
\midrule
Counting & $21.31 \pm 1.10$ & $20.76 \pm 1.90$ & $18.24 \pm 0.68$ & $27.25 \pm 3.46$ & $19.00 \pm 1.56$ \\
Llama 3.1 8B & $21.19 \pm 1.20$ & $21.00 \pm 2.20$ & $18.48 \pm 1.36$ & $26.67 \pm 3.49$ & $18.64 \pm 1.74$ \\
\textsc{Loci} & $\mathbf{12.47 \pm 0.79}$ & $\mathbf{12.48 \pm 1.59}$ & $\mathbf{11.10 \pm 1.00}$ & $\mathbf{16.96 \pm 2.24}$ & $\mathbf{9.36 \pm 1.22}$ \\
\bottomrule
\end{tabular}

%% file: tables/multiwoz_localization.tex
\begin{tabular}{lcccccc}
\toprule
 & $p=0.1$ & $p=0.2$ & $p=0.3$ & $p=0.4$ & $p=0.5$ & Average \\
\midrule
Failed draws & 247 & 427 & 600 & 740 & 844 & 2,858 \\
Mean $T$ & 2.58 & 2.56 & 2.55 & 2.51 & 2.49 & 2.54 \\
Mean $K$ & 1.13 & 1.15 & 1.31 & 1.36 & 1.53 & 1.29 \\
\midrule
\multicolumn{7}{l}{Hit@1} \\
\midrule
First & $0.457 \pm 0.032$ & $0.518 \pm 0.023$ & $0.523 \pm 0.021$ & $0.545 \pm 0.019$ & $0.609 \pm 0.016$ & $0.530 \pm 0.011$ \\
Last & $0.449 \pm 0.032$ & $0.436 \pm 0.023$ & $0.520 \pm 0.021$ & $0.549 \pm 0.019$ & $0.614 \pm 0.018$ & $0.513 \pm 0.012$ \\
Random & $0.459 \pm 0.010$ & $0.473 \pm 0.009$ & $0.529 \pm 0.009$ & $0.556 \pm 0.008$ & $0.625 \pm 0.009$ & $0.528 \pm 0.005$ \\
Counting & $0.608 \pm 0.031$ & $0.630 \pm 0.022$ & $0.647 \pm 0.020$ & $0.668 \pm 0.018$ & $0.723 \pm 0.015$ & $0.655 \pm 0.012$ \\
\textsc{Loci} & $\mathbf{0.673 \pm 0.030}$ & $\mathbf{0.730 \pm 0.022}$ & $\mathbf{0.772 \pm 0.018}$ & $\mathbf{0.792 \pm 0.016}$ & $\mathbf{0.825 \pm 0.013}$ & $\mathbf{0.758 \pm 0.011}$ \\
\midrule
\multicolumn{7}{l}{AP} \\
\midrule
First & $0.700 \pm 0.018$ & $0.738 \pm 0.013$ & $0.742 \pm 0.011$ & $0.753 \pm 0.010$ & $0.793 \pm 0.008$ & $0.745 \pm 0.006$ \\
Last & $0.699 \pm 0.018$ & $0.688 \pm 0.013$ & $0.737 \pm 0.012$ & $0.756 \pm 0.011$ & $0.793 \pm 0.009$ & $0.734 \pm 0.007$ \\
Random & $0.701 \pm 0.007$ & $0.709 \pm 0.006$ & $0.741 \pm 0.006$ & $0.757 \pm 0.005$ & $0.796 \pm 0.005$ & $0.741 \pm 0.004$ \\
Counting & $0.795 \pm 0.016$ & $0.803 \pm 0.012$ & $0.811 \pm 0.010$ & $0.824 \pm 0.009$ & $0.851 \pm 0.008$ & $0.817 \pm 0.006$ \\
\textsc{Loci} & $\mathbf{0.833 \pm 0.015}$ & $\mathbf{0.856 \pm 0.011}$ & $\mathbf{0.877 \pm 0.009}$ & $\mathbf{0.888 \pm 0.008}$ & $\mathbf{0.901 \pm 0.007}$ & $\mathbf{0.871 \pm 0.006}$ \\
\bottomrule
\end{tabular}

%% file: tables/multiwoz_recovery.tex
\begin{tabular}{@{}lrrrrrr@{}}
\toprule
Method & $p=0.1$ & $p=0.2$ & $p=0.3$ & $p=0.4$ & $p=0.5$ & Average \\
\midrule
Failed draws & 247 & 427 & 600 & 740 & 844 & 2,858 \\
\midrule
First + \texttt{VERIFY} & $3.07 \pm 0.09$ & $3.10 \pm 0.08$ & $3.56 \pm 0.08$ & $3.68 \pm 0.08$ & $4.11 \pm 0.09$ & $3.51 \pm 0.05$ \\
Last + \texttt{VERIFY} & $3.07 \pm 0.11$ & $3.23 \pm 0.08$ & $3.63 \pm 0.08$ & $3.68 \pm 0.08$ & $4.11 \pm 0.09$ & $3.54 \pm 0.06$ \\
Random + \texttt{VERIFY} & $3.12 \pm 0.10$ & $3.13 \pm 0.07$ & $3.59 \pm 0.08$ & $3.71 \pm 0.08$ & $4.12 \pm 0.09$ & $3.53 \pm 0.06$ \\
Counting + \texttt{VERIFY} & $2.81 \pm 0.08$ & $2.90 \pm 0.07$ & $3.40 \pm 0.08$ & $3.50 \pm 0.08$ & $4.01 \pm 0.09$ & $3.33 \pm 0.05$ \\
Restart & $3.15 \pm 0.10$ & $3.15 \pm 0.08$ & $3.18 \pm 0.07$ & $3.18 \pm 0.07$ & $3.19 \pm 0.07$ & $3.17 \pm 0.06$ \\
\midrule
\textsc{Loci} + \texttt{VERIFY} & $2.70 \pm 0.09$ & $2.80 \pm 0.08$ & $3.29 \pm 0.08$ & $3.38 \pm 0.08$ & $3.92 \pm 0.09$ & $3.22 \pm 0.05$ \\
\textsc{Loci} + \texttt{MASK} & $\mathbf{1.86 \pm 0.09}$ & $\mathbf{1.88 \pm 0.08}$ & $\mathbf{2.15 \pm 0.07}$ & $\mathbf{2.21 \pm 0.08}$ & $\mathbf{2.56 \pm 0.08}$ & $\mathbf{2.13 \pm 0.05}$ \\
\textsc{Loci} adaptive & $2.16 \pm 0.09$ & $2.31 \pm 0.08$ & $2.76 \pm 0.08$ & $2.80 \pm 0.07$ & $3.11 \pm 0.07$ & $2.63 \pm 0.06$ \\
\bottomrule
\end{tabular}

%% file: tables/multiwoz_recovery_components.tex
\begin{tabular}{@{}lrrrr@{}}
\toprule
Policy & Questions & Verifications & Wrong guesses & Total \\
\midrule
Always-\texttt{VERIFY} & $1.29 \pm 0.03$ & $1.59 \pm 0.02$ & $0.34 \pm 0.01$ & $3.22 \pm 0.05$ \\
Always-\texttt{MASK} & $1.66 \pm 0.04$ & $0.00 \pm 0.00$ & $0.47 \pm 0.02$ & $\mathbf{2.13 \pm 0.05}$ \\
Adaptive & $1.45 \pm 0.03$ & $0.94 \pm 0.03$ & $0.24 \pm 0.01$ & $2.63 \pm 0.06$ \\
\bottomrule
\end{tabular}

%% file: tables/reliability_localization.tex
\begin{tabular}{lccccc}
\toprule
 & $p=0.1$ & $p=0.2$ & $p=0.3$ & $p=0.4$ & $p=0.5$ \\
\midrule
Scored failures & 565 & 822 & 970 & 1,038 & 1,066 \\
\midrule
\multicolumn{6}{l}{Hit@1} \\
\midrule
Counting & $0.500 \pm 0.023$ & $0.530 \pm 0.022$ & $0.564 \pm 0.021$ & $0.628 \pm 0.022$ & $0.628 \pm 0.019$ \\
\textsc{Loci}, $\rho=1-p$ & $0.509 \pm 0.024$ & $0.528 \pm 0.021$ & $0.586 \pm 0.020$ & $0.674 \pm 0.021$ & $0.744 \pm 0.021$ \\
\textsc{Loci}, $\rho=0.5$ & $0.462 \pm 0.026$ & $0.534 \pm 0.023$ & $0.611 \pm 0.022$ & $0.710 \pm 0.020$ & $0.744 \pm 0.020$ \\
\textsc{Loci}, $\rho=0.7$ & $0.491 \pm 0.027$ & $0.532 \pm 0.022$ & $0.586 \pm 0.021$ & $0.660 \pm 0.020$ & $0.672 \pm 0.020$ \\
\textsc{Loci}, $\rho=0.9$ & $0.509 \pm 0.026$ & $0.535 \pm 0.022$ & $0.574 \pm 0.021$ & $0.634 \pm 0.019$ & $0.643 \pm 0.019$ \\
\textsc{Loci}, $\rho=0.95$ & $0.509 \pm 0.026$ & $0.535 \pm 0.022$ & $0.576 \pm 0.021$ & $0.631 \pm 0.020$ & $0.640 \pm 0.019$ \\
\midrule
\multicolumn{6}{l}{AP} \\
\midrule
Counting & $0.682 \pm 0.015$ & $0.686 \pm 0.014$ & $0.702 \pm 0.013$ & $0.748 \pm 0.012$ & $0.760 \pm 0.010$ \\
\textsc{Loci}, $\rho=1-p$ & $0.693 \pm 0.016$ & $0.685 \pm 0.014$ & $0.701 \pm 0.014$ & $0.755 \pm 0.013$ & $0.808 \pm 0.011$ \\
\textsc{Loci}, $\rho=0.5$ & $0.609 \pm 0.019$ & $0.660 \pm 0.016$ & $0.710 \pm 0.015$ & $0.776 \pm 0.012$ & $0.808 \pm 0.011$ \\
\textsc{Loci}, $\rho=0.7$ & $0.650 \pm 0.018$ & $0.672 \pm 0.015$ & $0.701 \pm 0.014$ & $0.752 \pm 0.013$ & $0.769 \pm 0.012$ \\
\textsc{Loci}, $\rho=0.9$ & $0.694 \pm 0.017$ & $0.696 \pm 0.014$ & $0.710 \pm 0.014$ & $0.749 \pm 0.012$ & $0.760 \pm 0.011$ \\
\textsc{Loci}, $\rho=0.95$ & $0.694 \pm 0.017$ & $0.697 \pm 0.014$ & $0.711 \pm 0.014$ & $0.749 \pm 0.012$ & $0.759 \pm 0.011$ \\
\bottomrule
\end{tabular}

%% file: tables/reliability_recovery.tex
\begin{tabular}{lcc}
\toprule
Assumed reliability & \textsc{Loci} + \texttt{VERIFY} & \textsc{Loci} adaptive \\
\midrule
$0.5$ & $12.17 \pm 0.64$ & $9.79 \pm 0.82$ \\
$0.7$ & $12.27 \pm 0.69$ & $9.01 \pm 0.76$ \\
$0.9$ & $12.32 \pm 0.68$ & $8.36 \pm 0.98$ \\
$0.95$ & $12.32 \pm 0.68$ & $8.86 \pm 1.02$ \\
\bottomrule
\end{tabular}